\documentclass[10pt,twocolumn,letterpaper]{article}

\usepackage[pagenumbers]{wacv} % arXiv / preprint with page numbers

\newcommand{\runpara}[1]{\par\vspace{0.35em}\noindent\textbf{#1}\enspace}

\usepackage{bm}
\usepackage{nicefrac}
\usepackage{amssymb}
\usepackage{amsthm}
\usepackage{stfloats}
\usepackage{multirow}
\newtheorem{proposition}{Proposition}
\providecommand{\cmark}{\checkmark}
\providecommand{\xmark}{--}

\newcommand{\sR}{\mathbb{R}}
\newcommand{\tH}{\bm{\mathcal{H}}}
\newcommand{\tS}{\bm{\mathcal{S}}}
\newcommand{\mK}{\mathbf{K}}
\newcommand{\mX}{\mathbf{X}}
\newcommand{\mP}{\mathbf{P}}

\newcommand{\vx}{\mathbf{x}}
\newcommand{\gL}{\mathcal{L}}
\DeclareMathOperator*{\argmax}{arg\,max}
\DeclareMathOperator*{\argmin}{arg\,min}

\definecolor{wacvblue}{rgb}{0.21,0.49,0.74}
\usepackage[pagebackref,breaklinks,colorlinks,allcolors=wacvblue]{hyperref}

\def\wacvPaperID{622} % Enter the WACV Paper ID here
\def\confName{WACV}
\def\confYear{2026}

\title{MEOM: Multi-View Expected-OKS Maximization \\for Human Pose Triangulation}

\author{Ziliang Xiong \quad Henglin Shi \quad Per-Erik Forss\'en\\
Computer Vision Laboratory, Department of Electrical Engineering,\\
Link\"oping University, Sweden\\
{\tt\small \{ziliang.xiong, henglin.shi, per-erik.forssen\}@liu.se}
}

\begin{document}
\maketitle
\begin{abstract}
Conventional algebraic triangulation solves 3D human pose estimation (HPE) from multi-view 2D keypoints.
The typical approach, decoding 2D keypoints from predicted heatmaps, is unreliable as heatmaps can be multimodal under occlusion, and collapsing them into single peaks discards their spatial distribution.
We seek to use the entire heatmap to estimate 3D poses more accurately, which requires solving two problems: how to robustly fuse heatmaps across views, and how to assess the reliability of heatmaps.
For the former, we introduce a novel objective, \textbf{M}ulti-view \textbf{E}xpected-\textbf{O}KS \textbf{M}aximization (\textbf{MEOM}), that locates a 3D joint where the views agree in probability mass.
For the latter, we adopt highest-density-region (HDR) calibration as a diagnostic of that mass, independently of distance-based metrics.
The proposed framework covers two settings, with and without 3D supervision.
Without 3D supervision, we optimize 3D poses from pretrained heatmap predictors by maximizing MEOM, achieving comparable performance with state-of-the-art methods that rely on larger backbones, temporal fusion, and simulated 3D data.
On ambiguous Human3.6M (H36MA) and occluded CMU Panoptic frames, the advantage is substantial.
When 3D labels are available, we train the model end-to-end with a combined MEOM and MSE loss,
achieving $19.11$\,mm absolute MPJPE on Human3.6M outperforming the state-of-the-art volumetric approach on absolute MPJPE at half the inference cost.
\end{abstract}

\section{Introduction}
\label{sec:intro}

\begin{figure}[!t]
  \centering
  \captionsetup{skip=10pt}
  \includegraphics[width=\linewidth,trim={0 30 0 50},clip]{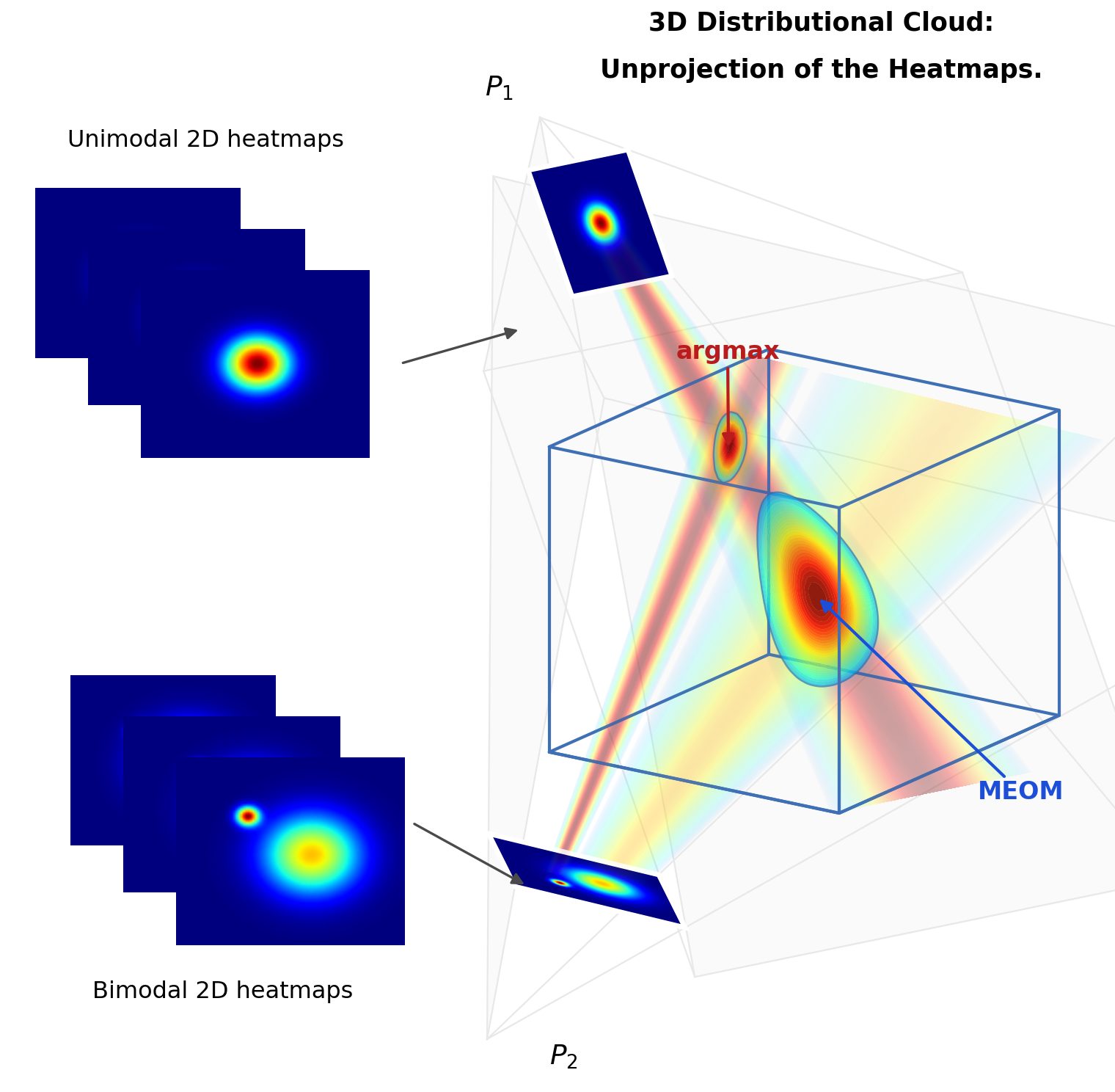}
  \caption{Multimodal heatmaps induce ambiguous triangulation.
  The upper view is unimodal; the lower view is bimodal, with a sharp outlier mode and a broader mode with most of the probability mass around.
  Our proposed MEOM selects a 3D point in the high-mass region of this implicit 3D density (illustration only).
  With argmax decoding, reprojection error and raw heatmap likelihood coincide, both selecting the sharp, low-mass intersection (red arrow).
  A 3D MSE loss instead collapses the implicit volume to isotropic Gaussians.}
  %\vspace{-5pt}
  \label{fig:intro-multimodal}
\end{figure}

Multi-view 3D HPE is classically solved by triangulation~\cite{hartley1997triangulation}: 2D keypoints are detected in each calibrated\footnote{Applied to a camera, ``calibrated'' means known intrinsics and extrinsics; applied to a predicted distribution, it means statistical reliability of its probability mass, which we make precise in Sec.~\ref{sec:method-hdr}.}
view, and the 3D joint is recovered by e.g., view-weighted direct linear transform (DLT)~\cite{hartley2003multiple}.
The geometry is explicit and the cost negligible, but every view is reduced to one point per joint: the spatial mass of each heatmap is collapsed before triangulation.
Volumetric methods keep that mass~\cite{joo2015panoptic,iskakov2019learnable,tu2020voxelpose}, but have cubic memory cost and learn a 3D prior tied to the capture geometry.
Fig.~\ref{fig:intro-multimodal} shows why fusion of collapsed peaks can be unreliable: each mode of a multimodal heatmap defines its own camera ray, so a point-based objective may settle on a sharp, low-mass intersection, as under ambiguity and occlusion.

Recent 2D HPE has moved away from heatmap collapse.
ProbPose~\cite{purkrabek2025probpose} treats each heatmap as a spatial distribution over the keypoint, calibrates it with temperature scaling, and reads the joint by Object Keypoint Similarity (OKS) decoding, which slightly improves 2D accuracy.
Similarly, we wish to use the entire heatmap for 3D HPE, which requires solving two problems: how to robustly fuse heatmaps across views into a 3D estimate, and how to assess whether the heatmap distribution is reliable.

For the former, we locate the 3D joint where the views agree in probability mass.
Each heatmap is convolved with a fixed OKS kernel, and a candidate 3D joint is scored by the aggregated response over cameras.
Maximizing that score is our proposed \textbf{M}ulti-view \textbf{E}xpected-\textbf{O}KS \textbf{M}aximization (\textbf{MEOM}).
The OKS kernel acts as a sliding window that seeks the largest probability mass rather than the highest peak. Thus, MEOM is a multi-view fusion method that optimizes 3D pose where the views agree in mass.

For the latter, we adopt highest-density-region (HDR) calibration~\cite{chung2024sampling} as a reliability diagnostic of mass distribution in heatmaps.
Because expected-OKS decoding reads and MEOM fuses probability mass, better-calibrated heatmaps are expected to result in more accurate triangulation.
HDR-ECE numerically summarizes the deviation from perfect HDR calibration.
Additionally, we identify that ProbPose-style map calibration is the discrete case of HDR calibration, and we further show that temperature scaling, which originally handles classification calibration~\cite{guo2017calibration}, locates the HDR calibration optimum, though it does not attain perfect calibration.

The MEOM score admits dual uses depending on the availability of 3D supervision.
When 3D supervision is not available, we freeze a pretrained heatmap predictor and optimize 3D joints: view-weighted DLT initialization is refined by maximizing MEOM.
When 3D supervision is available, we project the 3D labels to the predicted heatmaps and optimize them: after an MSE-only stage, a triangulator is trained end-to-end through differentiable expected-OKS decoding with a combined MEOM and 3D MSE loss, without subsequent MEOM 3D refinement.

We evaluate both settings on Human3.6M~\cite{ionescu2013human3}, its ambiguous subset H36MA~\cite{wehrbein2021probabilistic}, and CMU Panoptic~\cite{joo2015panoptic}.
When 3D supervision is not available, MEOM outperforms methods using reprojection error and heatmap likelihood as optimization targets, with larger gains on H36MA and occluded CMU Panoptic, and is comparable to methods that use larger backbones, temporal fusion, and simulated 3D data~\cite{upose3d2024}.
On CMU Panoptic, COCO-pretrained ProbPose with no target-domain 3D training remains competitive with the end-to-end algebraic baseline~\cite{iskakov2019learnable}.
When 3D supervision is available, MEOM surpasses state-of-the-art volumetric method~\cite{iskakov2019learnable} in absolute MPJPE at around half the inference cost (Tab.~\ref{tab:h36m-oks-mse}).\\

\paragraph{Contributions.}
\begin{enumerate}
    \item \textbf{MEOM score for algebraic triangulation.}
    We introduce MEOM for multi-view 3D HPE: a 3D candidate is scored by the aggregated Expected-OKS response across cameras, so fusion selects points on which the views agree in probability mass.

    \item \textbf{MEOM optimization without 3D labels.}
    With a frozen heatmap predictor, maximizing MEOM is competitive with methods that rely on larger backbones, temporal fusion, and simulated 3D data~\cite{upose3d2024}, without using 3D joint labels of the target dataset.

    \item \textbf{End-to-end MEOM training.}
    With 3D labels, an algebraic triangulator trained with expected-OKS decoding and a combined MEOM+MSE loss surpasses volumetric triangulation~\cite{iskakov2019learnable} in absolute MPJPE on Human3.6M at roughly half the inference cost.

    \item \textbf{HDR calibration for mass-based triangulation.}
    We compare quantile, HDR, and copula diagnostics under temperature scaling and measure how each predicts triangulation accuracy.
    On H36MA, over-sharp maps cost $9.6$\,mm Abs.\ MPJPE; the HDR-ECE-minimizing temperature is optimal for reprojection error and within $0.17$\,mm of the MEOM optimum.
\end{enumerate}

Sec.~\ref{sec:related} reviews related work.
Sec.~\ref{sec:method} presents MEOM, its two uses with and without 3D labels, and HDR calibration.
Sec.~\ref{sec:experiments} reports datasets, then frozen-network results, end-to-end training, and HDR diagnostics.
Sec.~\ref{sec:conclusion} concludes.

\section{Related Work}
\label{sec:related}

\begingroup
\renewcommand{\runpara}[1]{\par\vspace{0.1em}\noindent\textbf{#1}\enspace}
\runpara{3D multiview HPE}
Algebraic triangulation first detects 2D keypoints in each view and recovers 3D joints by a view-weighted DLT~\cite{hartley1997triangulation,hartley2003multiple}, either as post-hoc refinement under predicted 2D uncertainty~\cite{gundavarapu2019structured} or trained end-to-end through differentiable soft-argmax~\cite{iskakov2019learnable}.
It is fast and geometrically explicit, yet every heatmap is collapsed to one point and at most an additional scalar confidence.
Beyond point triangulation, Joo et al.~\cite{joo2015panoptic} discretize the 3D mocap space into a voxel grid and back-project each voxel center to all camera views, aggregating the corresponding heatmap scores into a 3D node-likelihood map by spatial voting to propose candidate 3D joints.
Subsequent work unprojects multi-view features into a voxel grid and fuses them with a 3D CNN, either for a single person with a learned body-shape prior~\cite{iskakov2019learnable} or for multi-person detection inside per-person cuboids~\cite{tu2020voxelpose}.
Voxel grids raise laboratory accuracy, but at cubic memory cost, and a 3D CNN trained on them can overfit to the capture geometry and camera layout.
Learned 2D-to-3D lifters further replace explicit triangulation with cross-view and temporal fusion networks~\cite{upose3d2024,chharia2025mvssm,rumpl2025ray}; strong benchmark numbers can still come with limited robustness when projective geometry is only implicit and training camera topologies are narrow.
Our method instead retains triangulation: Expected-OKS decoding and MEOM operate on predicted heatmaps without constructing a voxel grid or learning a 3D human prior.

\runpara{2D human pose estimation.}
Modern 2D HPE \cite{sun2019deep, chung2024sampling, xu2022vitpose, yuan2021hrformer, purkrabek2025probpose} localizes keypoints with dense heatmaps per joint, which represent the spatial uncertainty distribution over the keypoint.
HRNet~\cite{sun2019deep} maintains high-resolution representations throughout the network and remains a strong CNN baseline; ViTPose~\cite{xu2022vitpose} shows that plain vision transformers are competitive when trained at scale.
Both are typically supervised by pixelwise MSE against isotropic Gaussian heatmap targets, yet evaluated on COCO by mean Average Precision (mAP) under OKS~\cite{lin2014microsoft,gu2023calibration}.
OKS is a scale- and joint-normalized Gaussian similarity to the annotated keypoint.
OKS assigns a distinct per-joint scale (larger for hips, smaller for eyes), computed as the standard deviation of independent annotators repeatedly labeling the same people.
ProbPose~\cite{purkrabek2025probpose} instead predicts calibrated non-parametric heatmaps with temperature scaling, together with a presence score, and trains with OKS loss~\cite{maji2022yolo}.
We adopt the expected-OKS kernel to regularize each heatmap with the annotation noise, in both DLT initialization and MEOM fusion.
%We adopt its expected-OKS kernel, to convovle with each heatmap, as the building block for both DLT initialization and MEOM fusion.

\runpara{Uncertainty and calibration in pose estimation.}
In 3D HPE, aleatoric uncertainty has been modeled with structured Gaussians for weighted multi-view triangulation~\cite{gundavarapu2019structured}.
For monocular 2D-to-3D lifting, probabilistic methods represent ambiguous poses with normalizing flows~\cite{wehrbein2021probabilistic}, diffusion models~\cite{holmquist2023diffpose}, and flow matching~\cite{le2026fmpose}.
However, these methods do not directly evaluate the predicted distribution of 3D poses.
In 2D HPE, Gu et al.~\cite{gu2023calibration} show that standard confidences (e.g., heatmap maxima) are poorly calibrated to OKS, and propose a light post-hoc network to realign confidence with pose accuracy.
In our work, scalar confidence-to-OKS alignment is not enough: heatmaps must satisfy distributional calibration properties.
Bramlage et al.~\cite{bramlage2023plausible} instead regress parametric location uncertainties and apply post-hoc quantile recalibration~\cite{kuleshov2018accurate} independently per joint and axis.
Axis-wise quantile calibration is not enough as marginal calibration does not imply joint calibration~\cite{ziegel2014copula}.
Purkrabek et al.~\cite{purkrabek2025probpose} investigate heatmap calibration by temperature scaling~\cite{guo2017calibration}, yet do neither show the benefits of calibration nor theoretically prove its optimality.
A multivariate predictive distribution must first be reduced by a pre-rank that maps the forecast to a scalar~\cite{laajil2026calibrated}; different pre-ranks then induce quantile~\cite{kuleshov2018accurate}, HDR~\cite{chung2024sampling}, and copula~\cite{ziegel2014copula} calibration.
We therefore compare these diagnostics under inference-time temperature scaling and study how they correlate with multi-view triangulation accuracy.
%\enlargethispage{4\baselineskip}
\endgroup

\section{Method}
\label{sec:method}

\subsection{Motivation: why Expected OKS}
\label{sec:oks-theory}

Let $\mX\in\sR^3$ denote a 3D joint. Under skeletal and anthropometric constraints its distribution $p(\mX)$ lies on a low-dimensional, typically anisotropic manifold $\mathcal{M}\subset\sR^3$.
A calibrated camera with projection $\pi$ induces a heatmap as the projective marginal
\begin{equation}
  p(\vx)=\int_{\sR^3} p(\mX)\,\delta\big(\pi(\mX)-\vx\big)\,d\mX,
  \label{eq:proj-marginal}
\end{equation}
where $\vx\in\sR^2$ is the keypoint location in the image plane and $\delta$ is the Dirac delta that restricts the integral to the camera ray projecting to $\vx$, so $p(\vx)$ accumulates the 3D probability mass along that ray.
Training heatmaps with pixelwise MSE against a fixed isotropic Gaussian \cite{xu2022vitpose, sun2019deep} therefore misspecifies the marginal whenever uncertainty is elongated (depth ambiguity), truncated (crop boundary), or multimodal (occlusion).

Under a uniform spatial prior and conditional independence of views given $\mX$, with $v\in\{1,\ldots,V\}$ indexing the $V$ calibrated views and $I^{(v)}$ the image in view~$v$, the multi-view posterior satisfies
\begin{equation}
  \log p(\mX \mid \{I^{(v)}\})
  =
  \sum_{v=1}^{V}
  \log p\!\left(I^{(v)} \mid \mX\right)
  + C ,
  \label{eq:mv-loglik}
\end{equation}
where $C$ is a constant (See full derivation in Appendix~\ref{app:oks-mv-score}).
%A short formalization of this factorized score and its relation to an implicit field over $\mX$ appears in Appendix~\ref{app:oks-mv-score}.
Let $\tH^{(v)}$ denote the predicted spatial localization distribution (heatmap) on $\sR^{W\times H}$ in view~$v$, and let $\mP^{(v)}$ be the corresponding camera matrix with projection $\pi(\mP^{(v)},\cdot)$.
In practice, using $\tH^{(v)}$ directly in place of the view likelihood in~\eqref{eq:mv-loglik}, as a training loss or as a refinement objective, severely degrades performance due to sparsity.

To mitigate this, we follow ProbPose~\cite{purkrabek2025probpose} and convolve $\tH^{(v)}$ with an OKS kernel $\mK$ to obtain the expected-OKS response $\tS^{(v)}=\tH^{(v)}*\mK$.
The COCO OKS~\cite{lin2014microsoft} between a query location $\vx$ and a reference $\vx'$ for joint $k$ is given as
\begin{equation}
  \mathrm{OKS}_k(\vx,\vx')
  =
  \exp\!\left(
    -\frac{\|\vx-\vx'\|_2^{2}}{2 s^{2} \sigma_k^{2}}
  \right),
  \label{eq:oks-kernel}
\end{equation}
where $s$ is the subject scale and $\sigma_k$ is the fixed standard deviation for each keypoint.
%official per-keypoint constant.
When the kernel is discretized on a heatmap, the subject scale $s$ is implied by the activation-window resolution.
Then, we introduce the MEOM score by aggregating the expected-OKS responses of the 3D joint over views:
\begin{equation}
  \sum_{v=1}^{V}
  w^{(v)}\,
  \tS^{(v)}
  \!\Bigl(
    \pi\!\bigl(\mP^{(v)},\mX\bigr)
  \Bigr),
  \label{eq:meom-score}
\end{equation}
which is a practical surrogate for multi-view maximum a posteriori inference with finitely many views; here $w^{(v)}\ge 0$ is a view weight, usually predicted by the model.

\begin{figure*}[t]
  \centering
  \includegraphics[width=\linewidth,trim=0 0 0 40,clip]{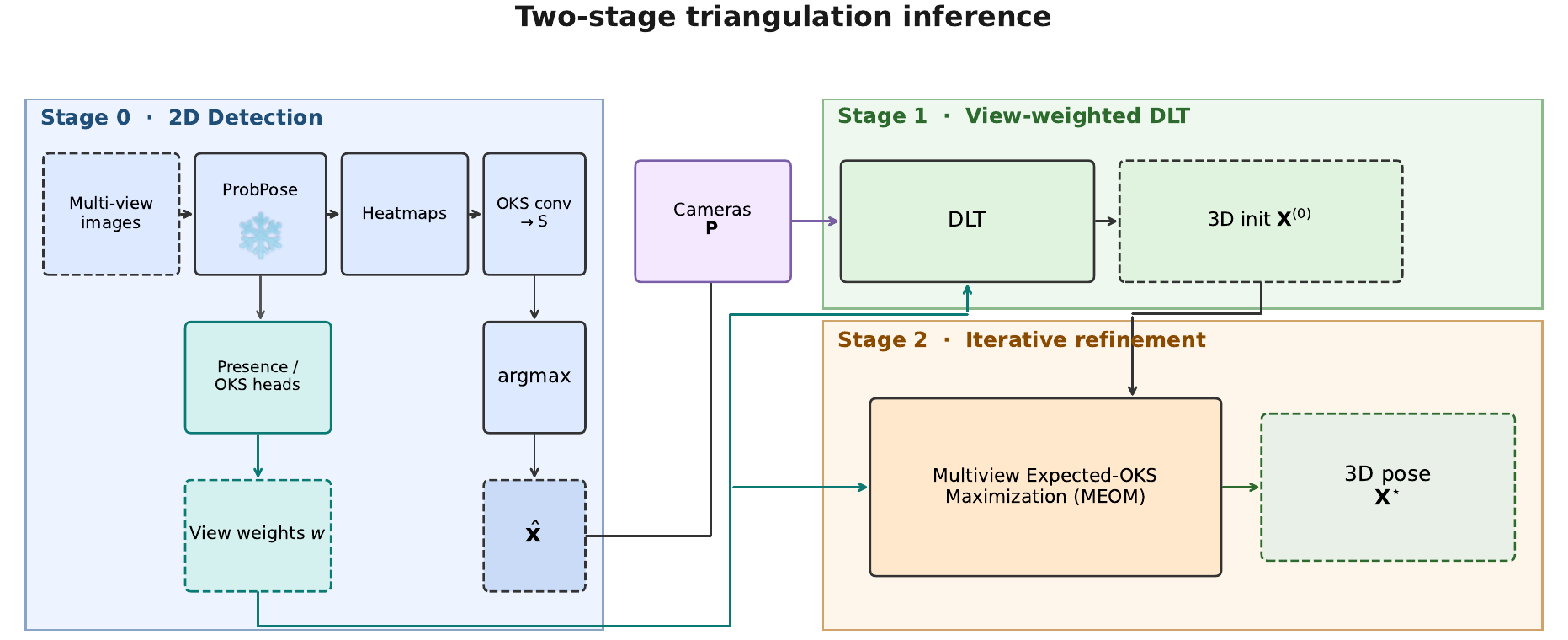}
  \vspace{-1.2em}
  \caption{MEOM triangulation without 3D labels.
  First, a frozen ProbPose~\cite{purkrabek2025probpose} produces per-view heatmaps and view weights; then view-weighted DLT gives an initial 3D pose estimation $\mX^{(0)}$, and iterative refinement finds $\mX^{*}$ that maximizes the MEOM score.}
  \label{fig:dlt-refine-inference}
\end{figure*}

The same score admits dual uses that differ in which variable is free: when 3D labels are unavailable, Sec.~\ref{sec:triangulation} freezes the heatmaps and finds $\mX^{*}$ that maximizes the MEOM score; when 3D labels are available, Sec.~\ref{sec:e2e} fixes $\mX=\mX^{\mathrm{gt}}$ and optimizes heatmaps (hence $\tS$) to raise that score.
Sec.~\ref{sec:method-hdr} further treats heatmap temperature and HDR calibration as a reliability diagnostic for both uses.

\subsection{MEOM triangulation without 3D labels}
\label{sec:triangulation}

This subsection studies triangulation with heatmaps fixed: 3D ground truth is typically available only in instrumented labs, so outside that setting a pretrained heatmap predictor is frozen and we optimize only the multi-view fusion of its detections via DLT initialization and iterative refinement. See the pipeline in Fig.~\ref{fig:dlt-refine-inference}.
We compare three algebraic optimization objectives for post-hoc refinement.
\par\vspace{\belowdisplayskip}
\noindent\textbf{Reprojection error.}
\label{sec:tri_reproj}
Minimize the weighted sum of reprojection residuals~\cite{gundavarapu2019structured}:
\begin{equation}
\begin{aligned}
    \mX_k^{*}
    &= \argmin_{\mX_k}\;
    \sum_{v=1}^{V} w_k^{(v)}\,
    \Bigl\|
      \hat{\vx}_k^{(v)}
      - \pi\!\bigl(\mP^{(v)}, \mX_k\bigr)
    \Bigr\|_2 ,
\end{aligned}
    \label{eq:tri_reproj}
\end{equation}
where indices $k$ and $v$ run over joints and calibrated views; $\mX_k$ is a candidate 3D joint, $\hat{\vx}_k^{(v)}$ the 2D network output, and $w_k^{(v)}$ the view weights.
\par\vspace{\belowdisplayskip}
\noindent\textbf{Heatmap likelihood.}
\label{sec:tri_heatmap}
Maximize the weighted localization likelihood of the projected point under the frozen heatmap $\tH_k^{(v)}$:
\begin{equation}
\begin{aligned}
    \mX_k^{*}
    &= \argmax_{\mX_k}\;
    \sum_{v=1}^{V}
    w_k^{(v)}\,
    \tH_k^{(v)}
    \Bigl(
      \pi\!\bigl(\mP^{(v)}, \mX_k\bigr)
    \Bigr) .
\end{aligned}
    \label{eq:tri_heatmap}
\end{equation}

\noindent\textbf{MEOM.}
The response $\tS$ of Sec.~\ref{sec:oks-theory} is shared by DLT initialization and MEOM refinement; $\hat{\vx}_k^{(v)}=\argmax \tS_k^{(v)}$ is used only to form the DLT initialization.
With heatmaps held fixed, refinement then maximizes Eq.~\eqref{eq:meom-score} over the 3D joint,
\begin{equation}
\begin{aligned}
    \mX_k^{*}
    &= \argmax_{\mX_k}\;
    \sum_{v=1}^{V}
    w_k^{(v)}\,
    \tS_k^{(v)}
    \Bigl(
      \pi\!\bigl(\mP^{(v)}, \mX_k\bigr)
    \Bigr) .
\end{aligned}
    \label{eq:tri_oks}
\end{equation}
In experiments, for all three objectives, we instantiate $w_k^{(v)}$ with either ProbPose presence probability or OKS prediction~\cite{purkrabek2025probpose} and compare both choices.

\subsection{End-to-end triangulation with MEOM}
\label{sec:e2e}

\begin{figure*}[t]
  \centering
  \includegraphics[width=\linewidth]{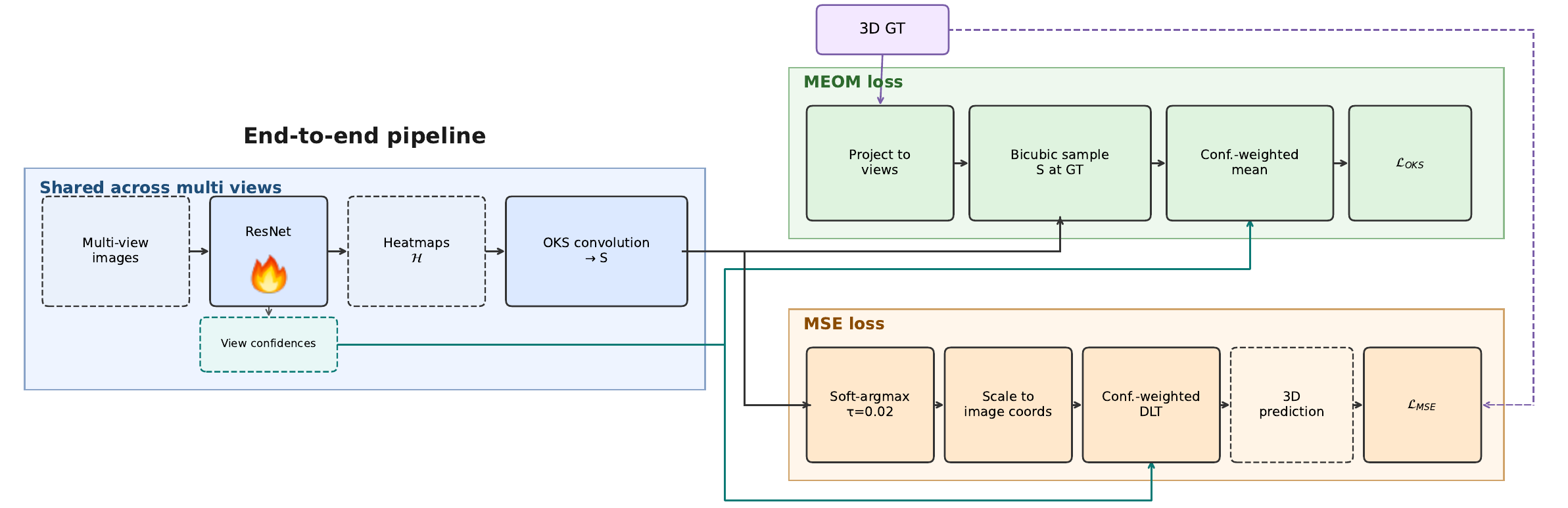}
  \vspace{-1.2em}
  \caption{End-to-end MEOM training.
  A shared backbone produces unit-scale heatmaps and view confidences.
  The MEOM loss $\gL_{\mathrm{MEOM}}$ samples $\tS$ at projections of the 3D ground truth and takes a confidence-weighted mean; the MSE loss $\gL_{\mathrm{MSE}}$ triangulates soft Expected-OKS detections by confidence-weighted DLT and compares to the same 3D ground truth.}
  \label{fig:forward-oks-mse}
\end{figure*}

Sec.~\ref{sec:triangulation} treats heatmaps as fixed observations and applies MEOM only at inference by searching over $\mX$.
When 3D supervision is available, we instead use the dual form of the MEOM score~\eqref{eq:meom-score}: fix $\mX=\mX^{\mathrm{gt}}$ and train the 2D network end-to-end to raise that score, jointly with a 3D MSE term that back-propagates through differentiable expected-OKS decoding and DLT (Fig.~\ref{fig:forward-oks-mse}).
Gradients from both losses flow into each view's heatmaps and act as a geometric pullback: probability mass is encouraged to concentrate where multi-view rays agree with the supervised 3D pose, without constructing an explicit voxel grid.

\subsubsection{Differentiable OKS decoding}
\label{sec:method-oks-soft}

Iskakov et al.~\cite{iskakov2019learnable} recover 2D joints by soft-argmax on raw heatmaps and then apply DLT.
Soft-argmax is convenient for training, but on multimodal heatmaps it returns the mass center that can fall between modes (expectation collapse).
We therefore keep the expected-OKS response, and decode it with temperatured-soft-argmax:
\begin{equation}
\begin{aligned}
  \hat{\vx}_k^{(v),\mathrm{soft}}
  &= \sum_{(x,y)}
    (x,y)\,
    \mathrm{softmax}
    \!\left(
      \frac{\tS_k^{(v)}}{\tau}
    \right)_{\!(x,y)} ,
\end{aligned}
  \label{eq:oks-soft}
\end{equation}
with a small temperature $\tau$.
As $\tau\to 0$, the soft weights concentrate on the dominant OKS peak, so the proxy tracks a mode rather than a broad center of mass on $\tH$ (We demonstrate this in Appendix~\ref{app:oks-mv-score}).
This soft path is used whenever 3D MSE must back-propagate through triangulation.

\subsubsection{Joint MSE and MEOM training}
\label{sec:method-oks-mse}

Under the projective marginal~\eqref{eq:proj-marginal}, a 3D MSE in \cite{iskakov2019learnable} on soft-decoded, triangulated joints is well matched only to a homoscedastic isotropic Gaussian prior on $\mX$; it misspecifies anisotropic or multimodal uncertainty.
As in Sec.~\ref{sec:oks-theory}, the two uses of the MEOM score~\eqref{eq:meom-score} differ in what is free.
%Without 3D ground truth, frozen-network MEOM searches over $\mX_k$ as in~\eqref{eq:tri_oks}.
With supervision, we fix $\mX_k=\mX_k^{\mathrm{gt}}$ and instead optimize the heatmaps and the predictor,
\begin{equation}
\begin{aligned}
    \{\tH_k^{(v)}\}^{*}
    &= \argmax_{\{\tH_k^{(v)}\}}\;
    \sum_{k}\sum_{v=1}^{V}
    w_k^{(v)}\,
    \tS_k^{(v)}
    \Bigl(
      \pi\!\bigl(\mP^{(v)}, \mX_k^{\mathrm{gt}}\bigr)
    \Bigr) .
\end{aligned}
    \label{eq:meom-train}
\end{equation}
%where $\tS_k^{(v)}=\tH_k^{(v)}*\mK_k$ and $w_k^{(v)}$ are predicted view weights.
Equivalently, we minimize the training loss
\begin{equation}
\begin{aligned}
  \gL_{\mathrm{MEOM}}
  &=
  -\sum_{k}\sum_{v=1}^{V}
  w_k^{(v)}\,
  \tS_k^{(v)}
  \Bigl(
    \pi\!\bigl(\mP^{(v)}, \mX_k^{\mathrm{gt}}\bigr)
  \Bigr) .
\end{aligned}
  \label{eq:meom-loss}
\end{equation}

Pure MEOM supervision can under-constrain absolute 3D scale and geometry. We therefore retain a 3D MSE term and train with the joint objective
\begin{equation}
\begin{aligned}
  \gL
  &= \lambda_{\mathrm{MEOM}}\,\gL_{\mathrm{MEOM}}
   + \lambda_{\mathrm{MSE}}\,\gL_{\mathrm{MSE}} ,
\end{aligned}
  \label{eq:joint-loss}
\end{equation}
where $\gL_{\mathrm{MSE}}$ is the smooth 3D keypoint MSE~\cite{iskakov2019learnable}. 
We train in two stages~\label{sec:method-two-stage}: Stage~1 uses MSE only ($\lambda_{\mathrm{MEOM}}=0$); Stage~2 fine-tunes with the joint loss~Eq.\eqref{eq:joint-loss}.
Unless stated otherwise, we omit subsequent MEOM 3D refinement after DLT so that reported gains are attributable to decoding and training rather than post-hoc optimization.

\subsection{Heatmap temperature and HDR calibration}
\label{sec:method-hdr}

Expected-OKS decoding and MEOM depend on the spatial mass of each heatmap, so poor calibration can distort triangulation.
A heatmap is a multivariate forecast, thus calibration requires a pre-rank that reduces it to a scalar~\cite{laajil2026calibrated}; different pre-ranks induce quantile~\cite{kuleshov2018accurate}, HDR~\cite{chung2024sampling}, and copula~\cite{ziegel2014copula} diagnostics.
Sec.~\ref{sec:results-hdr} compares all three under temperature scaling against triangulation accuracy.
Unlike distance-based metrics e.g., OKS~\cite{lin2014microsoft} and MPJPE, we adopt HDR-ECE because it scores the frequency of ground-truth keypoints falling in the highest-density region of the predicted mass at each level.
Similar to HDR, ProbPose states a requirement operationally for discrete heatmaps: the top $5\%$ of map mass should contain $5\%$ of ground-truth keypoints, the top $10\%$ should contain $10\%$, and so on.
We make this connection precise: ProbPose map calibration \emph{is} the discrete case of HDR calibration.

Following~\cite{chung2024sampling}, define the $\lambda$-density region of a predicted density $\hat{f}_{Y\mid x}$ by
$\mathrm{DR}_x(\lambda) := \{y : \hat{f}_{Y\mid x}(y) \geq \lambda\}$.
For each level $p\in(0,1)$, the highest-density region is the smallest such level set whose predictive mass is at least $p$,
\begin{equation}
\begin{aligned}
&\mathrm{HDR}_x(p) := \mathrm{DR}_x(\lambda^*) , \\
&\lambda^* = \sup \Bigl\{
      \lambda :
      P\!\bigl(
        \hat{Y} \in \mathrm{DR}_x(\lambda)
        \mid X{=}x
      \bigr)
      \ge p
    \Bigr\} .
\end{aligned}
\label{eq:hdr-def}
\end{equation}
HDR calibration requires
\begin{equation}
    P\bigl(Y \in \mathrm{HDR}_X(p)\bigr) = p
    \qquad\forall\, p \in (0, 1).
    \label{eq:hdr-calib}
\end{equation}

\begin{proposition}[ProbPose map calibration is HDR calibration]
\label{prop:probpose-hdr}
On a discrete spatial heatmap, sorting pixels by descending predicted density and taking the shortest prefix whose cumulative mass reaches a nominal level $p$ yields exactly $\mathrm{HDR}_x(p)$ in~\eqref{eq:hdr-def}.
Consequently, requiring that ground-truth keypoints fall in these prefixes at rate $p$ for all $p\in(0,1)$ is equivalent to~\eqref{eq:hdr-calib}.
\end{proposition}

\begin{proof}
Order the heatmap pixels so that $\hat{f}_{Y\mid x}(y_{(1)}) \ge \hat{f}_{Y\mid x}(y_{(2)}) \ge \cdots$.
For any threshold $\lambda$, $\mathrm{DR}_x(\lambda)$ is a prefix of this ordering.
The supremum $\lambda^*$ in~\eqref{eq:hdr-def} selects the shortest prefix whose predictive mass is at least $p$, which is the ``top-$p$ mass'' set used by ProbPose.
HDR calibration~\eqref{eq:hdr-calib} is therefore identical to requiring that these top-$p$ sets cover the ground truth with empirical frequency $p$.
\end{proof}

\paragraph{Temperature as a discrete HDR recalibrator.}
To improve heatmap calibration, ProbPose applies a temperature $T$ before spatial normalization, $\mathrm{softmax}(\tH^{(v)}/T)$, the post-hoc scalar that Guo et al.~\cite{guo2017calibration} use on classifiers.
Prop.~\ref{prop:probpose-hdr} makes~\eqref{eq:hdr-def} computable by sorting alone, so a heatmap needs neither density estimation nor continuity, and a one-dimensional sweep over $T$ is a cheaper surrogate for the learned mapping of Chung et al.~\cite{chung2024sampling} (Appendix~\ref{app:hdr}).
Flattening the softmax draws mass away from pixels denser than the ground truth, so per-level coverage error is V-shaped in $T$ and thus HDR-ECE has an interior minimum for each heatmap.
These per-map minima need not share an optimal $T$, but the network predictions tends be to over-or-underconfident in the same direction, so a shared scalar still yields a global HDR-ECE optimum (Appendix~\ref{app:hdr}).
The limitation is that temperature scaling never changes the density ordering thus cannot attain perfect calibraiton.
That shared $T$ is nevertheless the practical operating point for triangulation: over-sharp maps degrade 3D error, and the HDR-ECE-minimizing $T$ is optimal for reprojection-error refinement and lies on the MEOM accuracy plateau (Sec.~\ref{sec:results-hdr}).

\section{Experiments}
\label{sec:experiments}
\label{sec:results}

\subsection{Datasets and evaluation metrics}
\label{sec:datasets}

%on Human3.6M \cite{ionescu2013human3} and CMU Panoptic \cite{joo2015panoptic}
We follow the multi-view algebraic evaluation protocol in~\cite{iskakov2019learnable} by using \emph{absolute} MPJPE in global coordinates as the primary metric, and reporting pelvis-relative MPJPE and PA-MPJPE only as secondary.
Throughout the paper, expected-OKS decoding and MEOM use OKS kernels with the predefined per-joint constants $\sigma_k$ of the COCO protocol~\cite{lin2014microsoft}; retuning $\sigma_k$ for Human3.6M and CMU Panoptic may further improve accuracy, but it is outside the scope of this work.
\runpara{Human3.6M.} 
We test on subjects S9/S11 with the four standard cameras and the common 17-joint skeleton.
As in~\cite{iskakov2019learnable}, scenes with erroneous 3D annotations are removed from Human3.6M.
When 3D labels are unavailable (Sec.~\ref{sec:tri-ablation}), we subsample every 64th frame.
End-to-end training (Sec.~\ref{sec:results-oks-mse}) instead follows~\cite{iskakov2019learnable} by retaining every 5th frame with official lens undistortion.
We additionally evaluate on H36MA~\cite{wehrbein2021probabilistic}, an ambiguous subset of Human3.6M selected for occlusion and complex pose.
\runpara{CMU Panoptic.}
We use the Panoptic validation frame ranges and four-camera set of~\cite{iskakov2019learnable}.
Compared with Human3.6M, most CMU cameras do not see the full person for the entire sequence, leading to strong occlusions and missing body parts; the resulting incomplete or multimodal heatmaps require a robust objective to fuse across views.

\subsection{MEOM triangulation without 3D labels}
\label{sec:tri-ablation}

We evaluate the frozen-network pipeline of Sec.~\ref{sec:triangulation} on Human3.6M and CMU Panoptic.
Unlike Iskakov et al.~\cite{iskakov2019learnable}, we do \emph{not} undistort the images and project with the provided camera matrices as-is.
\runpara{Human3.6M.}
Because COCO and Human3.6M use different keypoint layouts, we fine-tune a COCO-pretrained ProbPose-small on Human3.6M 2D labels, then freeze it for triangulation.
We initialize with view-weighted DLT and refine for 80 Adam steps, using ProbPose OKS predictions as view weights at both stages.
Tab.~\ref{tab:h36m-upose3d} compares MEOM to multi-view methods that likewise use no 3D joint labels of the target dataset.
Tab.~\ref{tab:tri-every64} then ablates the refinement objective on Human3.6M and on H36MA under this frozen backbone.
See full ablation results in Appendix~\ref{app:tri-h36m-full}.

\begin{table}[t]
  \centering
  \caption{Performance of methods that do not use 3D joint labels on Human3.6M.
  DLT and UPose3D results are reported from~\cite{upose3d2024}, where absolute MPJPE is not reported.}
  \label{tab:h36m-upose3d}
  \scriptsize
  \setlength{\tabcolsep}{3pt}
  \resizebox{\linewidth}{!}{%
  \begin{tabular}{@{}llccccc@{}}
    \toprule
    Method & Backbone & Frames & MPJPE $\downarrow$ & Abs.\ $\downarrow$ & PA $\downarrow$ & Params (M) \\
    \midrule
    DLT~\cite{qiu2019cross} & ResNet-152 & 1 & $36.3$ & --- & --- & $68.6$ \\
    DLT~\cite{chen2018cascaded} & CPN & 1 & $30.5$ & --- & $27.6$ & $27.0$ \\
    UPose3D$^\dagger$~\cite{upose3d2024} & ResNet-152 & 1 & $31.0$ & --- & $29.0$ & $68.6$ \\
    UPose3D$^\dagger$~\cite{upose3d2024} & CPN & 1 & $26.9$ & --- & $24.1$ & $64.9$$^\ddagger$ \\
    UPose3D$^\dagger$~\cite{upose3d2024} & CPN & 27 & $26.4$ & --- & $23.4$ & $64.9$$^\ddagger$ \\
    MEOM (ours) & ProbPose-small (ViT-S) & 1 & $32.42$ & $36.04$ & $30.38$ & $24$ \\
    \bottomrule
  \end{tabular}%
  }

  \vspace{2pt}
  {\scriptsize $^\dagger$Trained on AMASS~\cite{mahmoodamass} with simulated 3D data.\\
  %$^\ddagger$UPose3D reports $64.87$\,M for their ResNet-152-based CPN.
  }
\end{table}

\begin{table}[t]
  \centering
  \caption{Triangulation results using three objectives without 3D labels on Human3.6M and H36MA~\cite{wehrbein2021probabilistic} ($N{=}6\,238$). Minimum per column in \textbf{bold}.}
  \label{tab:tri-every64}
  \label{tab:tri-h36ma}
  \scriptsize
  \setlength{\tabcolsep}{2.5pt}
  \resizebox{\linewidth}{!}{%
  \begin{tabular}{@{}lcccccc@{}}
    \toprule
    & \multicolumn{3}{c}{Human3.6M} & \multicolumn{3}{c}{H36MA} \\
    \cmidrule(lr){2-4}\cmidrule(lr){5-7}
    Method & MPJPE $\downarrow$ & Abs.\ $\downarrow$ & PA $\downarrow$ & MPJPE $\downarrow$ & Abs.\ $\downarrow$ & PA $\downarrow$ \\
    \midrule
    Init (DLT, unweighted) & 36.53 & 40.23 & 35.85 & 42.21 & 46.26 & 43.37 \\
    Init (DLT, OKS-weighted) & 34.27 & 37.94 & 32.71 & 37.12 & 41.28 & 36.17 \\
    \midrule
    Reprojection error & 34.69 & 38.11 & 32.89 & 37.84 & 41.70 & 36.63 \\
    Heatmap likelihood & 69.79 & 61.48 & 61.31 & 115.10 & 92.25 & 94.21 \\
    MEOM & \textbf{32.42} & \textbf{36.04} & \textbf{30.38} & \textbf{33.80} & \textbf{37.72} & \textbf{31.55} \\
    \bottomrule
  \end{tabular}%
  }
\end{table}

Against UPose3D~\cite{upose3d2024} (Tab.~\ref{tab:h36m-upose3d}), single-frame MEOM reaches $32.42$\,mm MPJPE ($36.04$\,mm absolute), close to their ResNet-152 result ($31.0$\,mm) and above DLT with the same backbone ($36.3$\,mm).
Their lower errors ($26.9$/$26.4$\,mm) come with temporal fusion and a compiler trained on AMASS~\cite{mahmoodamass} large-scale simulated 3D data, neither of which MEOM uses.
Among the three refinement objectives (Tab.~\ref{tab:tri-every64}), MEOM attains the lowest absolute MPJPE on both splits (36.04\,mm on Human3.6M; 37.72\,mm on H36MA), improving over OKS-weighted DLT and over reprojection error.
The gap is larger on H36MA, where ambiguous 2D detections make point-based fusion brittle: MEOM reduces absolute error by $4.0$\,mm relative to reprojection error, whereas heatmap-likelihood refinement collapses (92.25\,mm) and remains far worse than the initialization.
Raw heatmap likelihood is likewise weak on Human3.6M (61.48\,mm absolute): predicted heatmaps are sparse, with vanishing mass on most pixels, so maximizing $\tH$ at the projected location~\eqref{eq:tri_heatmap} stalls in flat zero-mass regions.

\runpara{ View-count sensitivity on H36MA.}
We subsample the first $V \in \{2,3,4\}$ cameras from the fixed four-camera setting on H36MA. 
Fig.~\ref{fig:h36ma-views-abs} reports absolute MPJPE versus $V$ (y-axis truncated above $60\,\mathrm{mm}$ for readability).
Full results are in Appendix~\ref{app:tri-h36m-full}.
Absolute error increases with fewer views; expected OKS remains the most accurate objective across all view counts.

\begin{figure}[t]
  \centering
  \includegraphics[width=0.85\linewidth,trim={0 0 0 22},clip]{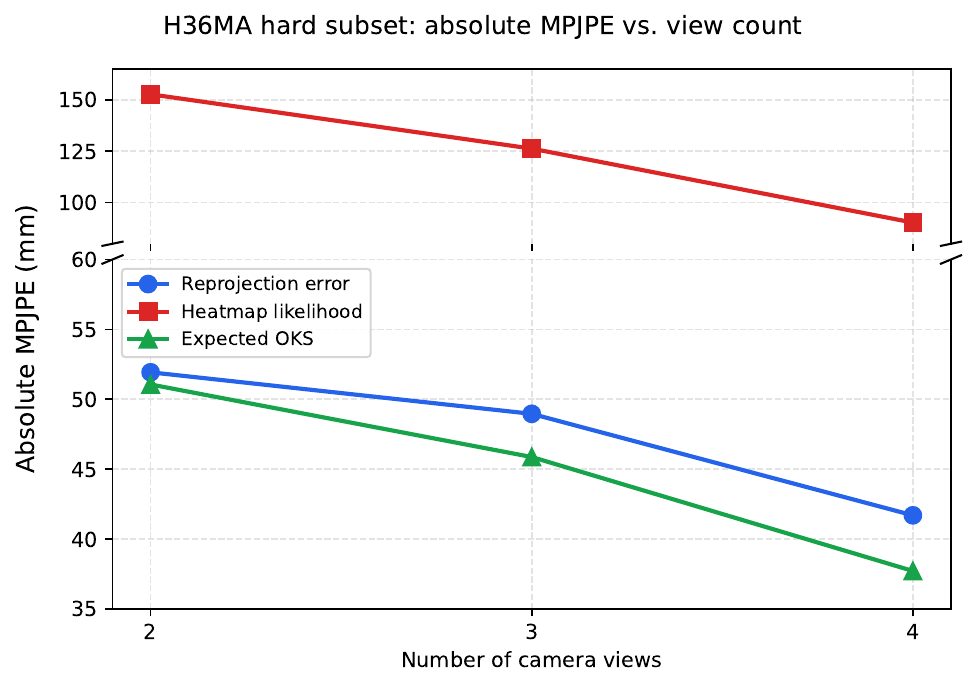}
  \caption{View-count sensitivity on H36MA: absolute MPJPE vs.\ number of camera views.
  Fixed setting per objective across 4/3/2 views (y-axis truncated above $60$\,mm).}
  \label{fig:h36ma-views-abs}
\end{figure}
\runpara{CMU Panoptic.}
\label{sec:cmu-tri-ablation}
On this set, 2D observations are frozen COCO-pretrained ProbPose heatmaps with \emph{no} CMU 2D or 3D fine-tuning; only algebraic DLT\,+\,refinement is evaluated.
Person crops use Mask R-CNN (MRCNN) detections following~\cite{iskakov2019learnable}.
Refinement objectives and view-weighting toggles match the Human3.6M protocol above.
Reported numbers from Iskakov et al.~\cite{iskakov2019learnable} are absolute MPJPE with four cameras (Tab.~3 in~\cite{iskakov2019learnable}).
Rankings among the three refinement objectives under the same frozen network are in Appendix~\ref{app:cmu-tri-mrcnn}.

\begin{table}[t]
  \centering
  \caption{Performance comparison on CMU Panoptic. Absolute MPJPE in millimeters. Results of Iskakov et al.\ are from~\cite{iskakov2019learnable} (Tab.~3). Minimum absolute MPJPE in \textbf{bold}.}
  \label{tab:cmu-lt-comparison}
  \footnotesize
  \setlength{\tabcolsep}{3pt}
  \begin{tabular}{@{}llcc@{}}
    \toprule
    Method & Backbone & 3D labels & Abs.\ MPJPE $\downarrow$ \\
    \midrule
    RANSAC~\cite{iskakov2019learnable} & \multirow{3}{*}{ResNet-152} & \multirow{3}{*}{\cmark} & 39.50 \\
    Algebraic~\cite{iskakov2019learnable} & & & 21.30 \\
    Volumetric~\cite{iskakov2019learnable} & & & \textbf{13.70} \\
    \midrule
    MEOM (ours) & ProbPose-small (ViT-S) & \xmark & 23.99 \\
    \bottomrule
  \end{tabular}
\end{table}

With a frozen COCO ProbPose backbone and \emph{no} CMU adaptation, MEOM reaches $23.99$\,mm absolute MPJPE---close to the algebraic baseline of Iskakov et al.~\cite{iskakov2019learnable} ($21.30$\,mm), which is trained end-to-end on the target domain, and far below their RANSAC triangulation ($39.50$\,mm).
This near-parity under a mismatched, frozen 2D network indicates that Expected-OKS fusion is already at the level of strong algebraic multi-view methods, and that MEOM remains effective when self-occlusion and multi-person occlusion corrupt per-view heatmaps.
The full ablation is in Appendix~\ref{app:cmu-tri-mrcnn}; Expected OKS again ranks first among the three refinement objectives under the frozen COCO ProbPose backbone.

\subsection{End-to-end triangulation with MEOM}
\label{sec:results-oks-mse}

We next train an algebraic triangulator end-to-end when 3D labels \emph{are} available (Sec.~\ref{sec:e2e}; $N{\approx}26\,610$).
The baselines are the official algebraic and volumetric checkpoints of Iskakov et al.~\cite{iskakov2019learnable}.
We use the same ResNet-152 backbone as their algebraic model (Tab.~\ref{tab:h36m-oks-mse}), replacing only their soft-argmax 2D decoding by soft Expected-OKS decoding~\eqref{eq:oks-soft} with $\tau{=}0.02$, followed by confidence-weighted DLT without post-hoc 3D refinement.
Training follows the two-stage schedule of Sec.~\ref{sec:method-two-stage}.
Tab.~\ref{tab:h36m-oks-mse} summarizes the main comparison; the volumetric and algebraic entries of Iskakov et al.~\cite{iskakov2019learnable} use the authors' official pretrained weights.

Stage~1 reaches $19.38$\,mm / $22.15$\,mm, already improving over the algebraic baseline, which indicates that simply switching to expected-OKS decoding is already beneficial.
Stage~2 with unfrozen confidences and the default MSE weight $\lambda_{\mathrm{MSE}}{=}100$ further improves to $19.11$\,mm absolute and $21.86$\,mm relative MPJPE---$0.06$\,mm below the official volumetric model on absolute error while remaining $0.86$\,mm relative above it.
The improvement brought by stage~2 over stage~1 shows that MEOM could further optimize the 3D joints.
A broader grid over $\lambda_{\mathrm{MSE}}$ (Appendix~\ref{app:h36m-mse-weight}) yields similar relative errors, indicating that our proposed loss remains stable.

These comparisons indicate that Expected-OKS-aligned algebraic training remains beneficial after lens undistortion, and that MSE+MEOM fine-tuning with an adaptable confidence head further improves over Stage~1, though volumetric aggregation still leads on relative MPJPE.
%Importantly, this gain does not come from lifting joints into a shared voxel grid or from a learned volumetric 3D prior of the kind used by Iskakov et al.~\cite{iskakov2019learnable}, which can overfit laboratory multi-camera setups; our pipeline stays in the continuous algebraic domain of projected OKS responses and DLT.
Stage~1 and Stage~2 use the same image-to-3D architecture at test time, so Tab.~\ref{tab:h36m-oks-mse} reports $158.2$\,GFLOPs for both, matching the official algebraic model ($158.0$) and remaining at roughly half the computation of volumetric triangulation ($301.4$).
The added Expected-OKS convolution is $0.162$\,GFLOPs---$919\times$ cheaper than unprojection and the voxel-to-voxel (V2V) 3D CNN ($149$\,GFLOPs); isolated-lift latency, memory, and the measurement protocol are in Appendix~\ref{app:flops-lift}.

\begin{table}[t]
  \centering
  \caption{Performance of \emph{end-to-end} algebraic triangulation on Human3.6M (every 5th frame). Absolute and root-relative MPJPE in millimeters. GFLOPs are fvcore inference multiply--adds. Best MPJPE in each error column in \textbf{bold}.}
  \label{tab:h36m-oks-mse}
  \footnotesize
  \setlength{\tabcolsep}{3pt}
  \begin{tabular}{@{}lccc@{}}
    \toprule
    Method & Abs. MPJPE\ $\downarrow$ & Rel. \ $\downarrow$ & GFLOPs $\downarrow$ \\
    \midrule
    LTHP (volumetric)~\cite{iskakov2019learnable}
      & 19.17 & \textbf{21.00} & 301.4 \\
    LTHP (algebraic)~\cite{iskakov2019learnable}
      & 19.89 & 22.59 & 158.0 \\
    \midrule
    Stage~1 (MSE, soft Exp-OKS)
      & 19.38 & 22.15 & 158.2 \\
    Stage~2 (MSE+MEOM)
      & \textbf{19.11} & 21.86 & 158.2 \\
    \bottomrule
  \end{tabular}
  \vspace{-1.2em}
\end{table}

\subsection{Temperature Scaling and HDR calibration}
\label{sec:results-hdr}
\label{sec:temp-sweep}
%\runpara{End-to-end ResNet HDR.}
%We next ask whether the unit-scale ResNet heatmaps produced by the two-stage models require a test-time temperature $T\neq 1$ for HDR calibration.
%On the same split we convert each joint heatmap to a spatial probability by $\mathrm{softmax}(\tH/T)$ \emph{before} OKS convolution and measure HDR-ECE against ground-truth projections.
%Fig.~\ref{fig:hdr-calib} compares Stage~1 and Stage~2 ($\lambda_{\mathrm{MSE}}{=}100$, unfrozen confidences) at $T{=}1$: Stage~1 attains a lower HDR-ECE ($0.012$ versus $0.025$), while Stage~2 remains reasonably calibrated.
%Sweeping the test-time heatmap multiplier does not change MPJPE under soft Expected-OKS decoding, which max-normalizes each map before the temperature soft-argmax; we therefore keep the training-time multiplier of $1$ at test time.
\runpara{Temperature sweep when 3D labels are unavailable.}
ProbPose's heatmap head has a temperature $T$ that controls heatmap calibration and sharpness.
We freeze the trained checkpoint and regenerate heatmaps at various $T$ on H36MA.
The full grid of calibration scores and the three refinement objectives are in Appendix~\ref{app:temp-h36ma}; here Fig.~\ref{fig:temp-hdr-argmin} reports HDR-ECE and NLL against MEOM and reprojection-error refinement.
\runpara{Better-calibrated map leads to better triangulation.}
On H36MA, HDR-ECE is unimodal in $T$ and lowest at $T{=}0.5$, while over-sharp maps degrade MEOM triangulation (Fig.~\ref{fig:temp-hdr-argmin}).
Reprojection-error refinement isolates the effect of calibration on decoding: it reads the heatmap only through the OKS-decoded 2D joint and the OKS view weights, and never evaluates the response map during 3D optimization.
Its absolute MPJPE is likewise unimodal in $T$ and attains its minimum exactly where HDR-ECE does, at $T{=}0.5$ (Fig.~\ref{fig:temp-hdr-argmin}, bottom left).
%The swing is only $0.73\,\mathrm{mm}$, and $T{=}0.5$ beats $T{=}0.75$ by just $0.01\,\mathrm{mm}$, so the evidence is the shared unimodal shape and the common argmin, not that one-sample margin.
This again shows that simply switching to expected-OKS decoding is already beneficial yet small.
MEOM reads probability mass rather than an isolated peak, so the same miscalibration is more costly: it peaks slightly later at $T{=}1.0$ but is within $0.17\,\mathrm{mm}$ of that optimum at the HDR-optimal $T{=}0.5$ (Fig.~\ref{fig:temp-hdr-argmin}, top left), i.e.\ inside a plateau spanning $T\in[0.5,1.5]$, while over-sharp maps cost up to $9.6\,\mathrm{mm}$.

\runpara{Which calibration score to trust.}
Relative to its own minimum, HDR-ECE spans an order of magnitude over the sweep, whereas NLL stays within about $2{\times}$ and is nearly flat near $T{=}0.75$ (Fig.~\ref{fig:temp-hdr-argmin}, right).
It shows that temperature scaling swiftly improves HDR calibration, which is more indicative to the mass distribution reliability than NLL.
Both scores place their optima on the MEOM plateau, but only HDR-ECE's sharp V coincides with the reprojection-error minimum at $T{=}0.5$.
In Appendix~\ref{app:temp-h36ma}, we further show that both marginal quantile calibration and copula calibration improve over large $T$ but fail to predict the MPJPE minimum for reprojection error and MEOM.

\begin{figure}[t]
  \centering
  \includegraphics[width=\linewidth,trim={0 0 0 10}]{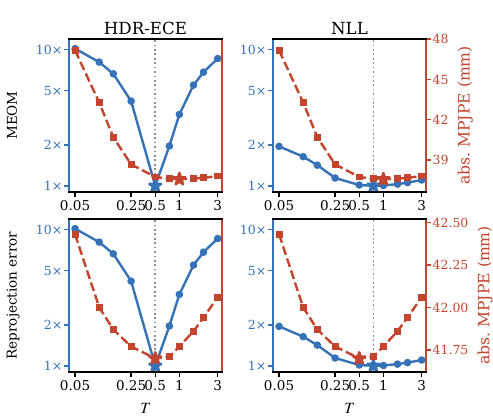}
  \vspace{-2.2em}
  \caption{Temperature scaling affects both heatmap calibration and triangulation accuracy.
    Calibration metrics relative to its minimum (blue, shared $1{\times}$--$10{\times}$ log scale) and refined absolute MPJPE (red) vs.\ ProbMapHead temperature $T$; stars mark each minimum. Reprojection error attains its optimum at the same $T{=}0.5$ that minimizes HDR-ECE; MEOM peaks at $T{=}1.0$, within $0.17\,\mathrm{mm}$ of its value at $T{=}0.5$.
  HDR-ECE spans an order of magnitude; NLL remains within about $2{\times}$ of its minimum.
  Note the different MPJPE ranges: $9.6\,\mathrm{mm}$ (MEOM) versus $0.73\,\mathrm{mm}$ (reprojection).}
  \label{fig:temp-hdr-argmin}
\end{figure}

\section{Conclusion}
\label{sec:conclusion}

We presented MEOM, an Expected-OKS objective that unifies 2D decoding and multi-view algebraic fusion: a 3D joint is inferred where the views agree in probability mass, without collapsing heatmaps to peaks or constructing a voxel grid.
With frozen 2D networks, maximizing MEOM outperforms point-based and likelihood refinement, especially under ambiguity and occlusion, and remains competitive with methods that use larger backbones, temporal fusion, and simulated 3D data.
When 3D labels are available, the dual form of the same score trains an algebraic triangulator that matches the volumetric method accuracy but at a substantially lower inference cost.

\enlargethispage{\baselineskip}
Because MEOM fuses probability mass, poorly calibrated heatmaps degrade triangulation.
A temperature that minimizes HDR-ECE is a practical operating point for both reprojection-error and MEOM refinement; NLL is less indicative of that mass reliability near the same optimum.
We therefore recommend HDR calibration as a first-class diagnostic for mass-based multi-view pose, to be reported alongside MPJPE.

{
    \small
    \bibliographystyle{ieeenat_abbrev}
    \bibliography{main}

@String(CVPR= {IEEE Conf. Comput. Vis. Pattern Recog.})

@String(ICCV= {Int. Conf. Comput. Vis.})

@String(ECCV= {Eur. Conf. Comput. Vis.})

@inproceedings{iskakov2019learnable,
  title={Learnable triangulation of human pose},
  author={Iskakov, Karim and Burkov, Egor and Lempitsky, Victor and Malkov, Yury},
  booktitle={Proceedings of the IEEE/CVF International Conference on Computer Vision},
  pages={7718--7727},
  year={2019}
}

@inproceedings{gundavarapu2019structured,
  title={Structured aleatoric uncertainty in human pose estimation},
  author={Gundavarapu, Nitesh B and Srivastava, Divyansh and Mitra, Rahul and Sharma, Abhishek and Jain, Arjun},
  booktitle={Proceedings of the IEEE/CVF Conference on Computer Vision and Pattern Recognition Workshops},
  year={2019}
}

@inproceedings{purkrabek2025probpose,
  title={Probpose: A probabilistic approach to 2d human pose estimation},
  author={Purkrabek, Miroslav and Matas, Jiri},
  booktitle={Proceedings of the Computer Vision and Pattern Recognition Conference},
  pages={27124--27133},
  year={2025}
}

@inproceedings{wehrbein2021probabilistic,
  title={Probabilistic monocular 3d human pose estimation with normalizing flows},
  author={Wehrbein, Tom and Rudolph, Marco and Rosenhahn, Bodo and Wandt, Bastian},
  booktitle={Proceedings of the IEEE/CVF International Conference on Computer Vision},
  pages={11199--11208},
  year={2021}
}

@inproceedings{holmquist2023diffpose,
  title={DiffPose: Multi-hypothesis Human Pose Estimation using Diffusion Models},
  author={Holmquist, Karl and Wandt, Bastian},
  booktitle={Proceedings of the IEEE/CVF International Conference on Computer Vision (ICCV)},
  pages={15977--15987},
  year={2023}
}

@article{le2026fmpose,
  title={Flow Matching for Probabilistic Monocular 3D Human Pose Estimation},
  author={Le, Cuong and Melnyk, Pavlo and Wandt, Bastian and Wadenb{\"a}ck, M{\aa}rten},
  journal={arXiv preprint arXiv:2601.16763},
  year={2026},
  note={Accepted to TMLR}
}

@inproceedings{chung2024sampling,
  title={Sampling-based multi-dimensional recalibration},
  author={Chung, Youngseog and Char, Ian and Schneider, Jeff},
  booktitle={Forty-first International Conference on Machine Learning},
  year={2024}
}

@article{gu2023calibration,
  title={On the Calibration of Human Pose Estimation},
  author={Gu, Kerui and Chen, Rongyu and Yao, Angela},
  journal={arXiv preprint arXiv:2311.17105},
  year={2023}
}

@inproceedings{joo2015panoptic,
  title={Panoptic Studio: A Massively Multiview System for Social Motion Capture},
  author={Joo, Hanbyul and Liu, Hao and Tan, Lei and Gui, Lin and Nabbe, Bart and Matthews, Iain and Kanade, Takeo and Nobuhara, Shohei and Sheikh, Yaser},
  booktitle={Proceedings of the IEEE International Conference on Computer Vision (ICCV)},
  pages={3334--3342},
  year={2015}
}

@inproceedings{upose3d2024,
  title={UPose3D: Uncertainty-Aware 3D Human Pose Estimation with Cross-View and Temporal Cues},
  author={Davoodnia, Vandad and Ghorbani, Saeed and Carbonneau, Marc-Andr{\'e} and Messier, Alexandre and Etemad, Ali},
  booktitle={European Conference on Computer Vision (ECCV)},
  year={2024}
}

@inproceedings{chharia2025mvssm,
  title={MV-SSM: Multi-View State Space Modeling for 3D Human Pose Estimation},
  author={Chharia, Aviral and Gou, Wenbo and Dong, Haoye},
  booktitle={Proceedings of the Computer Vision and Pattern Recognition Conference (CVPR)},
  pages={11590--11599},
  year={2025}
}

@article{rumpl2025ray,
  title={RUMPL: Ray-Based Transformers for Universal Multi-View 2D to 3D Human Pose Lifting},
  author={Ghasemzadeh, Seyed Abolfazl and Alahi, Alexandre and De Vleeschouwer, Christophe},
  journal={arXiv preprint arXiv:2512.15488},
  year={2025}
}

@article{ziegel2014copula,
  author  = {Ziegel, Johanna F. and Gneiting, Tilmann},
  title   = {Copula calibration},
  journal = {Electronic Journal of Statistics},
  volume  = {8},
  number  = {2},
  pages   = {2619--2638},
  year    = {2014},
  doi     = {10.1214/14-EJS964}
}

@inproceedings{laajil2026calibrated,
  title     = {Calibrated Multivariate Distributional Regression with Pre-Rank Regularization},
  author    = {Laajil, Aya and Zhalieva, Elnura and Desobry, Naomi and Ben Taieb, Souhaib},
  booktitle = {AISTATS 2026 Workshop on Towards Trustworthy Predictions: Theory and Applications of Calibration for Modern AI},
  year      = {2026},
  address   = {Tangier, Morocco},
  url       = {https://openreview.net/forum?id=249afd28cc1189658c3297ee834bb5235d05b693}
}

@inproceedings{kuleshov2018accurate,
  title={Accurate uncertainties for deep learning using calibrated regression},
  author={Kuleshov, Volodymyr and Fenner, Nathan and Ermon, Stefano},
  booktitle={International conference on machine learning},
  pages={2796--2804},
  year={2018},
  organization={PMLR}
}

@inproceedings{sun2019deep,
  title={Deep high-resolution representation learning for human pose estimation},
  author={Sun, Ke and Xiao, Bin and Liu, Dong and Wang, Jingdong},
  booktitle={2019 IEEE/CVF conference on computer vision and pattern recognition (CVPR)},
  pages={5686--5696},
  year={2019},
  organization={IEEE}
}

@article{xu2022vitpose,
  title={Vitpose: Simple vision transformer baselines for human pose estimation},
  author={Xu, Yufei and Zhang, Jing and Zhang, Qiming and Tao, Dacheng},
  journal={Advances in neural information processing systems},
  volume={35},
  pages={38571--38584},
  year={2022}
}

@inproceedings{lin2014microsoft,
  title     = {Microsoft {COCO}: Common Objects in Context},
  author    = {Lin, Tsung-Yi and Maire, Michael and Belongie, Serge and Hays, James and Perona, Pietro and Ramanan, Deva and Doll{\'a}r, Piotr and Zitnick, C. Lawrence},
  booktitle = {European Conference on Computer Vision (ECCV)},
  year      = {2014}
}

@article{ionescu2013human3,
  title={Human3. 6m: Large scale datasets and predictive methods for 3d human sensing in natural environments},
  author={Ionescu, Catalin and Papava, Dragos and Olaru, Vlad and Sminchisescu, Cristian},
  journal={IEEE transactions on pattern analysis and machine intelligence},
  volume={36},
  number={7},
  pages={1325--1339},
  year={2013},
  publisher={IEEE}
}

@article{yuan2021hrformer,
  title={Hrformer: High-resolution vision transformer for dense predict},
  author={Yuan, Yuhui and Fu, Rao and Huang, Lang and Lin, Weihong and Zhang, Chao and Chen, Xilin and Wang, Jingdong},
  journal={Advances in neural information processing systems},
  volume={34},
  pages={7281--7293},
  year={2021}
}

@inproceedings{tu2020voxelpose,
  title={Voxelpose: Towards multi-camera 3d human pose estimation in wild environment},
  author={Tu, Hanyue and Wang, Chunyu and Zeng, Wenjun},
  booktitle={European conference on computer vision},
  pages={197--212},
  year={2020},
  organization={Springer}
}

@inproceedings{bramlage2023plausible,
  title={Plausible uncertainties for human pose regression},
  author={Bramlage, Lennart and Karg, Michelle and Curio, Crist{\'o}bal},
  booktitle={2023 IEEE/CVF International Conference on Computer Vision (ICCV)},
  pages={15087--15096},
  year={2023},
  organization={IEEE}
}

@inproceedings{guo2017calibration,
  title={On calibration of modern neural networks},
  author={Guo, Chuan and Pleiss, Geoff and Sun, Yu and Weinberger, Kilian Q},
  booktitle={International conference on machine learning},
  pages={1321--1330},
  year={2017},
  organization={PMLR}
}

@article{hartley1997triangulation,
  title={Triangulation},
  author={Hartley, Richard I. and Sturm, Peter},
  journal={Computer Vision and Image Understanding},
  volume={68},
  number={2},
  pages={146--157},
  year={1997}
}

@book{hartley2003multiple,
  title={Multiple View Geometry in Computer Vision},
  author={Hartley, Richard and Zisserman, Andrew},
  edition={2},
  publisher={Cambridge University Press},
  year={2003}
}

@inproceedings{mahmoodamass,
  title={Amass: Archive of motion capture as surface shapes. In 2019 IEEE},
  author={Mahmood, Naureen and Ghorbani, Nima and Troje, Nikolaus F and Pons-Moll, Gerard and Black, Michael},
  booktitle={CVF International Conference on Computer Vision (ICCV)},
  pages={5441--5450},
  year={2019}
}

@inproceedings{qiu2019cross,
  title={Cross view fusion for 3d human pose estimation},
  author={Qiu, Haibo and Wang, Chunyu and Wang, Jingdong and Wang, Naiyan and Zeng, Wenjun},
  booktitle={2019 IEEE/CVF International Conference on Computer Vision (ICCV)},
  pages={4341--4350},
  year={2019},
  organization={IEEE}
}

@inproceedings{chen2018cascaded,
  title={Cascaded pyramid network for multi-person pose estimation},
  author={Chen, Yilun and Wang, Zhicheng and Peng, Yuxiang and Zhang, Zhiqiang and Yu, Gang and Sun, Jian},
  booktitle={2018 IEEE/CVF Conference on Computer Vision and Pattern Recognition},
  pages={7103--7112},
  year={2018},
  organization={IEEE}
}

@inproceedings{maji2022yolo,
  title={Yolo-pose: Enhancing yolo for multi person pose estimation using object keypoint similarity loss},
  author={Maji, Debapriya and Nagori, Soyeb and Mathew, Manu and Poddar, Deepak},
  booktitle={2022 IEEE/CVF Conference on Computer Vision and Pattern Recognition Workshops (CVPRW)},
  pages={2636--2645},
  year={2022},
  organization={IEEE}
}
}

\maketitlesupplementary
\appendix
\section{Multi-view OKS as a factorized score}
\label{app:oks-mv-score}

This appendix records notation shared by Sec.~\ref{sec:triangulation} in the main paper, the Expected-OKS decoding formulas, and a short probabilistic reading of multi-view fusion used in Sec.~\ref{sec:oks-theory} in the main paper.
The likelihood discussion is a modeling justification for finite $V$, not a claim that infinite cameras reconstruct the true 3D density by tomographic inversion.

\paragraph{Notation.}
\label{app:notation}
Unless stated otherwise, $k\in\{1,\ldots,K\}$ indexes joints and $v\in\{1,\ldots,V\}$ indexes views.
We write $\mX_k\in\sR^3$ for a candidate 3D joint position, with DLT initialization $\mX^{(0)}$ and refined estimate $\mX^{*}$ as in Fig.~\ref{fig:dlt-refine-inference} in the main paper; $\hat{\vx}_k^{(v)}\in\sR^2$ for the 2D network output in view~$v$ (Fig.~\ref{fig:dlt-refine-inference} in the main paper: $\hat{\vx}$); $\mP^{(v)}$ for the camera matrix and $\pi(\mP^{(v)},\cdot)$ for its projection; $w_k^{(v)}$ for the view weight in DLT and refinement; $\tH_k^{(v)}$ for the predicted heatmap on $\sR^2$; and $\tS_k^{(v)}$ for the Expected-OKS response map used by MEOM.

\paragraph{Factorized multi-view score.}
Let $\mX\in\sR^3$ be a 3D joint and write $\mathcal{I}=\{I^{(v)}\}_{v=1}^{V}$ for the calibrated views.
Bayes' rule gives the posterior
\begin{equation}
  p(\mX \mid \mathcal{I})
  =
  \frac{p(\mathcal{I}\mid\mX)\,p(\mX)}{p(\mathcal{I})} .
  \label{eq:bayes-3d}
\end{equation}
Conditional independence of the views given $\mX$ factorizes the joint image likelihood into per-view terms,
\begin{equation}
  p(\mathcal{I}\mid\mX)
  =
  \prod_{v=1}^{V} p\!\left(I^{(v)}\mid\mX\right) .
  \label{eq:view-factor}
\end{equation}
Under a uniform spatial prior, $p(\mX)$ does not depend on $\mX$ in the region of interest, so taking logarithms yields Eq.~\eqref{eq:mv-loglik} in the main paper:
\begin{equation}
\begin{aligned}
  \log p(\mX \mid \mathcal{I})
  &=
  \sum_{v=1}^{V}
  \log p\!\left(I^{(v)} \mid \mX\right)
  + C ,
\end{aligned}
\end{equation}
where $C=\log p(\mX)-\log p(\mathcal{I})$ collects the $\mX$-independent prior and evidence.
Modeling
\begin{equation}
  p\!\left(I^{(v)}\mid\mX\right)
  \propto
  \tS_k^{(v)}
  \!\Bigl(
    \pi\!\bigl(\mP^{(v)},\mX\bigr)
  \Bigr)
\end{equation}
for the joint of interest yields the MEOM aggregation maximized in Eq.~\eqref{eq:tri_oks} in the main paper
(up to monotone transforms and optional view weights $w_k^{(v)}$).
Equivalently, one may maximize
\begin{equation}
\begin{aligned}
  &\sum_{v=1}^{V}
  \log
  \tS_k^{(v)}
  \!\Bigl(
    \pi\!\bigl(\mP^{(v)},\mX\bigr)
  \Bigr)
\end{aligned}
\end{equation}
as an implicit continuous field over $\mX$ without voxelization; our method evaluates related objectives via 2D OKS decoding and algebraic DLT/refinement rather than dense 3D search.

\paragraph{OKS kernel.}
\label{app:oks-kernel}
The COCO OKS~\cite{lin2014microsoft} between a query location $\vx$ and a reference $\vx'$ for joint $k$ is given by Eq.~\eqref{eq:oks-kernel} in the main paper:
\begin{equation}
  \mathrm{OKS}_k(\vx,\vx')
  =
  \exp\!\left(
    -\frac{\|\vx-\vx'\|_2^{2}}{2 s^{2} \sigma_k^{2}}
  \right),
  \label{eq:oks-kernel-re}
\end{equation}
where $s$ is the subject scale and $\sigma_k$ is the official per-keypoint constant (larger for hips and shoulders, smaller for eyes and wrists).
We use the COCO $\sigma_k$ without retuning; when the kernel is discretized on a heatmap, $s$ is the subject scale implied by the activation-window resolution.

\paragraph{Expected-OKS response and hard decoding.}
\label{app:oks-decode}
Following ProbPose~\cite{purkrabek2025probpose}, let $p_L^{(v)}$ be the localization distribution on the activation window $AW$.
For each query pixel $(x, y) \in AW$, the Expected-OKS response is the expectation of Eq.~\eqref{eq:oks-kernel-re} under $p_L^{(v)}$ 
\begin{equation}
\begin{aligned}
    \tS_k^{(v)}(x, y)
    &= \sum_{(x', y') \in AW}
      p_L^{(v)}(x', y') \\
    &\qquad\cdot\,
      \mathrm{OKS}_k\bigl((x, y),\, (x', y')\bigr),
\end{aligned}
    \label{eq:oks_map}
\end{equation}
or equivalently $\tS_k^{(v)}=\tH_k^{(v)}*\mK_k$ with the discretized kernel $\mK_k(\vx)=\mathrm{OKS}_k(\vx,\mathbf{0})$,
\begin{equation}
  \tS_k^{(v)} = \tH_k^{(v)} * \mK_k ,
  \label{eq:oks-conv}
\end{equation}
Hard decoding reads
\begin{equation}
  \hat{\vx}_k^{(v)} = \argmax_{(x,y)} \tS_k^{(v)}(x,y) ,
  \label{eq:oks-argmax}
\end{equation}
The response $\tS_k^{(v)}$ is shared by DLT initialization and by MEOM (Eq.~\eqref{eq:tri_oks} in the main paper) in the no 3D supervision pipeline; hard-decoded $\hat{\vx}_k^{(v)}$ is used only to initialize DLT.
Soft Expected-OKS decoding in Sec.~\ref{sec:method-oks-soft} in the main paper likewise operates on $\tS_k^{(v)}$.

\paragraph{Temperature soft-argmax on the OKS response.}
The soft Expected-OKS decoding in Eq.~\eqref{eq:oks-soft} in the main paper applies $\mathrm{softmax}(\tS_k^{(v)}/\tau)$ on the OKS response.
If two modes of $\tS_k^{(v)}$ differ by a margin $\epsilon>0$, their soft weight ratio scales as $\mathrm{e}^{\epsilon/\tau}$.
As $\tau\to 0$, secondary modes are exponentially suppressed and the soft mean approaches $\argmax \tS_k^{(v)}$, i.e., a differentiable mode tracker on the OKS map.
Operating on $\tS_k^{(v)}=\tH_k^{(v)}*\mK_k$ rather than on raw $\tH_k^{(v)}$ further attenuates single-pixel spikes before this selection, which mitigates (but does not categorically eliminate) expectation collapse under multimodality.
Figures~\ref{fig:oks-decoding} and~\ref{fig:oks-soft-argmax} illustrate these contrasts on synthetic $64{\times}48$ heatmaps.

\begin{figure}[t]
  \centering
  \includegraphics[width=\linewidth]{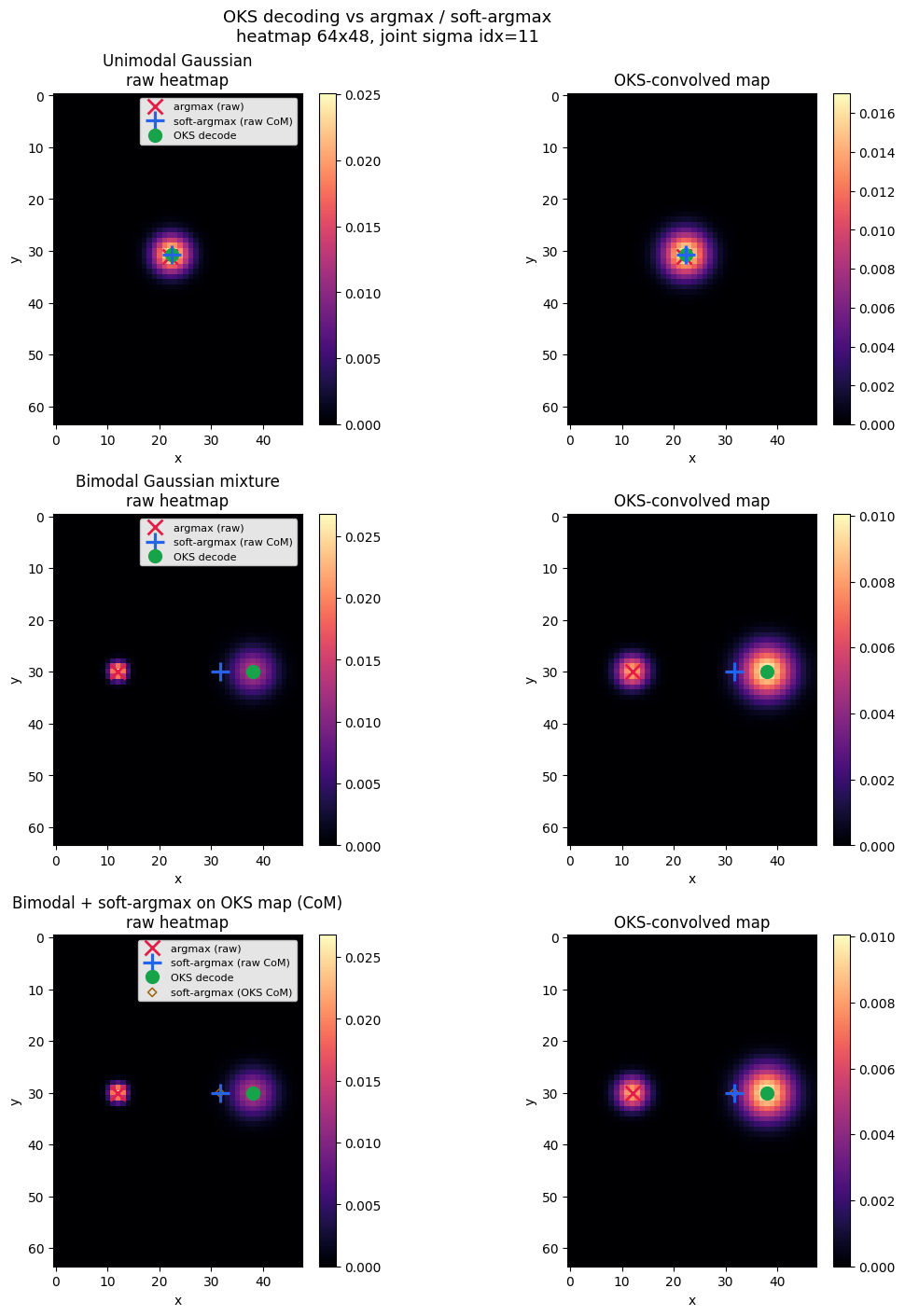}
  \caption{OKS decoding vs.\ argmax and soft-argmax on synthetic heatmaps ($64{\times}48$).
  Left: raw heatmap $\tH$; right: OKS-convolved response $\tS=\tH*\mK$.
  Markers show $\argmax$ on the raw map (red~$\times$), soft-argmax / center of mass on the raw map (blue~$+$), hard OKS decode $\argmax\tS$ (green~$\circ$), and---in the bottom row---soft-argmax on $\tS$ (brown~$\diamond$).
  First row: On a unimodal Gaussian, all estimators coincide.
  Second row: Hard OKS decoding selects the mass-dominant right mode from a bimodal mixture, while raw $\argmax$ prefers the sharper left peak, and raw soft-argmax collapses between modes; 
  Third row: Soft-argmax with $\tau{=}1$ on $\tS$ coincides with Soft-argmax with $\tau{=}1$ on raw $\tH$.}
  \label{fig:oks-decoding}
\end{figure}

\begin{figure}[t]
  \centering
  \includegraphics[width=\linewidth, trim={0 0 0 40}, clip]{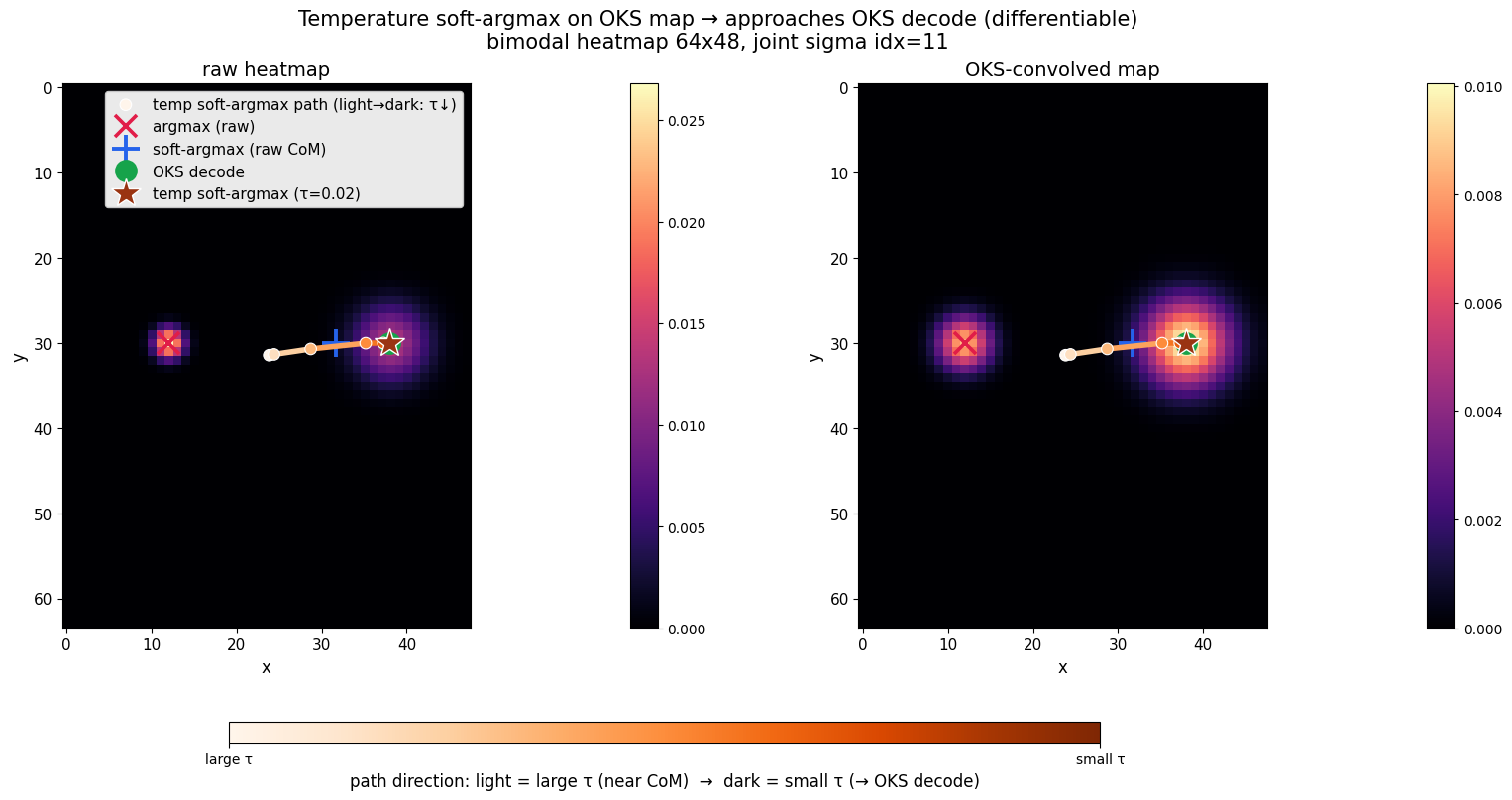}
  \caption{Temperature soft-argmax on a bimodal heatmap approaches hard OKS decode.
  Left: raw $\tH$; right: OKS response $\tS$.
  As temperature $\tau$ decreases (light$\to$dark along the orange path), soft-argmax moves from a broad center of mass toward a mode.
  On $\tS$, the operating point $\tau{=}0.02$ (red star) coincides with hard OKS decode (green~$\circ$), providing a differentiable proxy used in Eq.~\eqref{eq:oks-soft} in the main paper.
  Soft-argmax on raw $\tH$ remains farther from the OKS mode even at the same $\tau$.}
  \label{fig:oks-soft-argmax}
\end{figure}

%\paragraph{Relation to volumetric fusion.}
%Explicit volumetric methods push features into a discretized 3D grid and fuse there~\cite{iskakov2019learnable}.
%The factorized OKS score instead keeps fusion in the continuous domain of projected 2D responses.
%We do not assert that $V\to\infty$ yields a unique Radon reconstruction of $p(\mX)$; with the small $V$ used in practice, MEOM is best viewed as a geometry-aware surrogate likelihood whose empirical benefit is measured in Sec.~\ref{sec:experiments} in the main paper.

\section{Temperature Scaling as a Discrete HDR Recalibrator}
\label{app:hdr}

Sec.~\ref{sec:method-hdr} in the main paper defines highest-density regions and shows that ProbPose map calibration is HDR calibration.
Here we prove the monotonicity claimed there---raising $T$ increases the HDR coverage for all levels $p \in (0,1)$, so a one-dimensional sweep locates the HDR-ECE optimum and serves as a cheaper surrogate for the learned mapping of Chung et al.~\cite{chung2024sampling}, and then bound what that single parameter can achieve, since minimizing this error and attaining exact HDR calibration are requirements at two different levels.

In~\cite{purkrabek2025probpose}, temperature scaling~\cite{guo2017calibration} is used as a post-hoc sharpness knob on a discrete heatmap.
Write $h_i$ for the predicted score of pixel $i\in\{1,\dots,N\}$ (assumed distinct), $g$ for the ground-truth pixel, and
\begin{equation}
A := \{i : h_i > h_g\},
\qquad
U(T) := \sum_{i\in A}\pi_i(T),
\label{eq:hdr-level}
\end{equation}
with $\pi_i(T)\propto\exp(h_i/T)$, so that by Prop.~\ref{prop:probpose-hdr} in the main paper the ground truth lies in $\mathrm{HDR}_x(p)$ if and only if $U(T)<p$.
With $\beta=1/T$, $Z=\sum_i\exp(\beta h_i)$ and $Z_A=\sum_{i\in A}\exp(\beta h_i)$, we have $U=Z_A/Z$ and
\begin{equation}
\frac{\mathrm{d}}{\mathrm{d}\beta}\log\frac{Z_A}{Z}
= \mathbb{E}_A[h]-\mathbb{E}[h] > 0
\label{eq:temp-deriv}
\end{equation}
whenever $A\neq\emptyset$, since every score in $A$ exceeds every score outside it.
Raising $T$ therefore decreases $U$ on every sample simultaneously, from $U\to 1$ as $T\to 0^{+}$ to $U\to|A|/N$ as $T\to\infty$: flattening the softmax draws mass away from the pixels denser than the ground truth, so the share of mass lying above it falls.
Write $\widehat{c}(p;T)$ for the empirical coverage at level $p$, the fraction of samples with $U(T)<p$.
Since every $U$ decreases in $T$, a sample once counted stays counted, so $\widehat{c}(p;\cdot)$ is non-decreasing and crosses the level $p$ at most once.
The per-level calibration error is $\lvert\widehat{c}(p;T)-p\rvert$, an absolute value composed with a monotone function of $T$, hence V-shaped: it falls while $\widehat{c}(p;T)<p$, vanishes at the crossing, and rises once $\widehat{c}(p;T)>p$.
HDR-ECE averages these V-shaped curves over levels whose crossings need not coincide, so monotonicity guarantees an interior optimum reachable by a one-dimensional sweep, though not a unique one in general; on our data it is unique, at $T{\approx}0.5$ (Sec.~\ref{sec:results-hdr} in the main paper).
Per-map crossings need not agree, but modern networks tend to be overconfident or underconfident in the same direction~\cite{guo2017calibration} when tested on the same dataset, so those crossings cluster and the dataset operating interval stays tight.
Thus, a global temperature learned on the validation set can minimize HDR-ECE on the test set, which is why we report HDR-ECE alongside 3D error.

%These two statements live at different levels, and neither implies the other.
%On a single heatmap, $T$ controls $U$ alone, and Eq.~\eqref{eq:temp-deriv} lets it be placed anywhere in $(|A|/N,1)$; the ordering it slides along, however, is frozen, because $t\mapsto\exp(t/T)$ is strictly increasing and hence leaves $A$---and the decoded keypoint---unchanged.
Temperature scaling is not without limitations.
Exact HDR calibration (Eq.~\eqref{eq:hdr-calib} in the main paper) is instead a distributional requirement on the whole dataset, $U\sim\mathrm{Unif}(0,1)$, which asks one shared scalar to satisfy every level at once when the per-sample sets $A$ are fixed and mutually inconsistent.
Minimizing HDR-ECE is thus attainable while uniformity in general is not: a temperature corrects sharpness, but no $T$ moves a ground truth past a competing mode.
Post-hoc calibration of a trained heatmap is bounded by this, which is why we report HDR-ECE alongside 3D error rather than treating a recalibrated map as fully calibrated.
\section{ProbPose fine-tuning on Human3.6M}
\label{app:probpose-ft}

COCO and Human3.6M both use 17 keypoints, but the joint sets and channel orders differ: COCO is organized around face and limb landmarks, whereas Human3.6M follows a kinematic tree with a pelvis root, spine, thorax, and neck.
A COCO-pretrained heatmap head is therefore not aligned with Human3.6M annotations.
We start from the official COCO-pretrained ProbPose-small checkpoint and fine-tune it on the standard Human3.6M training subjects S1, S5, S6, S7, and S8, so that the output heatmaps follow the Human3.6M joint order.
Subjects S9 and S11 are held out for testing; bbox-normalized PCK@0.1 on this split peaks at $0.976$ by epoch~5.

\section{Additional experimental tables}
\label{app:extra-results}

This appendix collects triangulation tables, the end-to-end compute protocol, and temperature-sweep results deferred from Sec.~\ref{sec:experiments} in the main paper.

\subsection{Human3.6M triangulation without 3D labels: full presence/OKS grids}
\label{app:tri-h36m-full}

Tabs.~\ref{tab:tri-every64-full} and~\ref{tab:tri-h36ma-full} expand Tabs.~\ref{tab:tri-every64} in the main paper with presence-probability vs.\ predicted OKS view weighting and DLT~/~refinement ON/OFF toggles.
Presence weighting barely moves reprojection error, while Expected OKS remains best under OKS weights in both splits.

\begin{table}[t]
  \centering
  \caption{Full triangulation ablation on Human3.6M S9/S11 every 64th frame ($N{=}2\,181$). Presence and OKS view weights at DLT~/~refinement. Minimum per column in \textbf{bold}.}
  \label{tab:tri-every64-full}
  \scriptsize
  \setlength{\tabcolsep}{2.5pt}
  \resizebox{\linewidth}{!}{%
  \begin{tabular}{@{}llccc@{}}
    \toprule
    Objective & DLT / Ref. & MPJPE $\downarrow$ & Abs.\ $\downarrow$ & PA $\downarrow$ \\
    \midrule
    Init (DLT only) & \xmark / \xmark & 36.53 & 40.23 & 35.85 \\
    Init (DLT only) & OKS / \xmark & 34.27 & 37.94 & 32.71 \\
    \midrule
    \textit{Reprojection error} & Presence (\xmark/\xmark) & 35.76 & 39.21 & 34.38 \\
    & Presence (\xmark/\cmark) & 35.76 & 39.21 & 34.37 \\
    & Presence (\cmark/\xmark) & 35.76 & 39.21 & 34.38 \\
    & Presence (\cmark/\cmark) & 35.76 & 39.21 & 34.37 \\
    & OKS (\xmark/\xmark) & 35.76 & 39.21 & 34.38 \\
    & OKS (\xmark/\cmark) & 34.67 & 38.10 & 32.88 \\
    & OKS (\cmark/\xmark) & 35.78 & 39.23 & 34.39 \\
    & OKS (\cmark/\cmark) & 34.69 & 38.11 & 32.89 \\
    \midrule
    \textit{Heatmap likelihood} & Presence (\xmark/\xmark) & 69.76 & 61.65 & 62.07 \\
    & Presence (\xmark/\cmark) & 69.78 & 61.66 & 62.12 \\
    & Presence (\cmark/\xmark) & 69.84 & 61.73 & 62.18 \\
    & Presence (\cmark/\cmark) & 69.84 & 61.72 & 62.15 \\
    & OKS (\xmark/\xmark) & 69.76 & 61.65 & 62.07 \\
    & OKS (\xmark/\cmark) & 71.93 & 62.13 & 62.51 \\
    & OKS (\cmark/\xmark) & 68.35 & 60.93 & 60.62 \\
    & OKS (\cmark/\cmark) & 69.79 & 61.48 & 61.31 \\
    \midrule
    \textit{Expected OKS} & OKS (\xmark/\xmark) & 33.12 & 36.81 & 31.51 \\
    & OKS (\xmark/\cmark) & 32.75 & 36.39 & 30.95 \\
    & OKS (\cmark/\xmark) & 32.74 & 36.42 & 30.83 \\
    & OKS (\cmark/\cmark) & \textbf{32.42} & \textbf{36.04} & \textbf{30.38} \\
    \bottomrule
  \end{tabular}%
  }
\end{table}

\begin{table}[t]
  \centering
  \caption{Full triangulation ablation on H36MA ($N{=}6\,238$). Presence and OKS view weights at DLT~/~refinement. Minimum per column in \textbf{bold}.}
  \label{tab:tri-h36ma-full}
  \scriptsize
  \setlength{\tabcolsep}{2.5pt}
  \resizebox{\linewidth}{!}{%
  \begin{tabular}{@{}llccc@{}}
    \toprule
    Objective & DLT / Ref. & MPJPE $\downarrow$ & Abs.\ $\downarrow$ & PA $\downarrow$ \\
    \midrule
    Init (DLT only) & \xmark / \xmark & 42.21 & 46.26 & 43.37 \\
    Init (DLT only) & OKS / \xmark & 37.12 & 41.28 & 36.17 \\
    \midrule
    \textit{Reprojection error} & Presence (\xmark/\xmark) & 40.36 & 44.17 & 40.20 \\
    & Presence (\xmark/\cmark) & 40.36 & 44.17 & 40.19 \\
    & Presence (\cmark/\xmark) & 40.36 & 44.17 & 40.20 \\
    & Presence (\cmark/\cmark) & 40.36 & 44.17 & 40.19 \\
    & OKS (\xmark/\xmark) & 40.36 & 44.17 & 40.20 \\
    & OKS (\xmark/\cmark) & 37.80 & 41.67 & 36.61 \\
    & OKS (\cmark/\xmark) & 40.40 & 44.19 & 40.23 \\
    & OKS (\cmark/\cmark) & 37.84 & 41.70 & 36.63 \\
    \midrule
    \textit{Heatmap likelihood} & Presence (\xmark/\xmark) & 112.94 & 90.36 & 94.49 \\
    & Presence (\xmark/\cmark) & 112.85 & 90.31 & 94.45 \\
    & Presence (\cmark/\xmark) & 112.82 & 90.32 & 94.40 \\
    & Presence (\cmark/\cmark) & 112.86 & 90.30 & 94.42 \\
    & OKS (\xmark/\xmark) & 112.94 & 90.36 & 94.49 \\
    & OKS (\xmark/\cmark) & 116.99 & 91.90 & 95.48 \\
    & OKS (\cmark/\xmark) & 111.62 & 90.58 & 92.94 \\
    & OKS (\cmark/\cmark) & 115.10 & 92.25 & 94.21 \\
    \midrule
    \textit{Expected OKS} & OKS (\xmark/\xmark) & 35.29 & 39.26 & 34.08 \\
    & OKS (\xmark/\cmark) & 34.67 & 38.47 & 32.92 \\
    & OKS (\cmark/\xmark) & 34.28 & 38.31 & 32.30 \\
    & OKS (\cmark/\cmark) & \textbf{33.80} & \textbf{37.72} & \textbf{31.55} \\
    \bottomrule
  \end{tabular}%
  }
\end{table}

\subsection{CMU Panoptic triangulation ablation}
\label{app:cmu-tri-mrcnn}

Tab.~\ref{tab:cmu-tri-mrcnn-bbox} expands the CMU comparison of Sec.~\ref{sec:cmu-tri-ablation} in the main paper with the full objective and view-weighting grid under MRCNN boxes.
The absolute MPJPE $23.99$\,mm reported in the main text is the Expected-OKS row with OKS weighting at both DLT and refinement.

\begin{table}[t]
  \centering
  \caption{Triangulation ablation on CMU Panoptic validation~\cite{iskakov2019learnable} --- MRCNN bounding boxes ($N=8\,110$). Pre-computed Mask R-CNN detections from~\cite{iskakov2019learnable}. Four views per sample. View-weight columns list DLT / refinement (ON = \cmark, OFF = \xmark). Init: unweighted and OKS-weighted DLT (no refinement). Minimum per metric column in \textbf{bold}.}
  \label{tab:cmu-tri-mrcnn-bbox}
  \scriptsize
  \setlength{\tabcolsep}{2.5pt}
  \resizebox{\linewidth}{!}{%
  \begin{tabular}{@{}llccc@{}}
    \toprule
    Objective & DLT / Ref. & MPJPE $\downarrow$ & Abs.\ $\downarrow$ & PA $\downarrow$ \\
    \midrule
    Init (DLT only) & \xmark / \xmark & 50.74 & 37.18 & 42.45 \\
    Init (DLT only) & OKS / \xmark & 44.86 & 28.52 & 31.41 \\
    \midrule
    \textit{Reprojection error} & Presence (\xmark/\xmark) & 49.10 & 35.65 & 41.47 \\
     & Presence (\xmark/\cmark) & 49.06 & 35.60 & 41.39 \\
     & Presence (\cmark/\xmark) & 49.10 & 35.64 & 41.45 \\
     & Presence (\cmark/\cmark) & 49.05 & 35.58 & 41.37 \\
     & OKS (\xmark/\xmark) & 49.10 & 35.65 & 41.47 \\
     & OKS (\xmark/\cmark) & 46.14 & 31.00 & 35.70 \\
     & OKS (\cmark/\xmark) & 49.18 & 35.82 & 41.59 \\
     & OKS (\cmark/\cmark) & 46.05 & 31.15 & 35.70 \\
    \midrule
    \textit{Heatmap likelihood} & Presence (\xmark/\xmark) & 128.11 & 113.65 & 138.97 \\
     & Presence (\xmark/\cmark) & 128.01 & 113.54 & 138.89 \\
     & Presence (\cmark/\xmark) & 127.74 & 113.24 & 138.65 \\
     & Presence (\cmark/\cmark) & 127.72 & 113.22 & 138.64 \\
     & OKS (\xmark/\xmark) & 128.11 & 113.65 & 138.97 \\
     & OKS (\xmark/\cmark) & 133.89 & 119.25 & 145.01 \\
     & OKS (\cmark/\xmark) & 128.05 & 113.62 & 140.53 \\
     & OKS (\cmark/\cmark) & 133.96 & 119.57 & 147.10 \\
    \midrule
    \textit{Expected OKS} & OKS (\xmark/\xmark) & 42.80 & 24.95 & 26.47 \\
     & OKS (\xmark/\cmark) & 42.71 & 24.39 & 25.82 \\
     & OKS (\cmark/\xmark) & \textbf{42.29} & 24.40 & 25.64 \\
     & OKS (\cmark/\cmark) & 42.33 & \textbf{23.99} & \textbf{25.18} \\
    \bottomrule
  \end{tabular}%
  }
\end{table}

\subsection{End-to-end MSE weight ablation}
\label{app:h36m-mse-weight}

Tab.~\ref{tab:h36m-mse-weight} expands Tab.~\ref{tab:h36m-oks-mse} in the main paper with Stage~2 joint MSE+MEOM fine-tuning under MEOM weight $1.0$ and $\lambda_{\mathrm{MSE}}\in\{10,50,100,200,500\}$.
Relative MPJPE is best at $\lambda_{\mathrm{MSE}}{=}100$; absolute MPJPE is marginally best at $\lambda_{\mathrm{MSE}}{=}500$.

\begin{table}[t]
  \centering
  \caption{Human3.6M end-to-end Stage~2 MSE weight ablation on the undistorted protocol of Sec.~\ref{sec:datasets} in the main paper (every 5th frame). Absolute and root-relative MPJPE in millimeters (lower is better). Soft Expected-OKS decoding ($\tau{=}0.02$); MEOM weight $1.0$. Official baselines and Stage~1 repeated from Tab.~\ref{tab:h36m-oks-mse}. Best algebraic numbers in each column are in \textbf{bold}.}
  \label{tab:h36m-mse-weight}
  \footnotesize
  \setlength{\tabcolsep}{3.5pt}
  \begin{tabular}{@{}lcc@{}}
    \toprule
    Method & Abs.\ MPJPE $\downarrow$ & Rel.\ MPJPE $\downarrow$ \\
    \midrule
    Iskakov et al.\ (volumetric)~\cite{iskakov2019learnable}
      & $19.17$ & $21.00$ \\
    Iskakov et al.\ (algebraic)~\cite{iskakov2019learnable}
      & $19.89$ & $22.59$ \\
    \midrule
    hm$\times$1 Stage~1 (MSE, soft Exp-OKS)
      & $19.38$ & $22.15$ \\
    hm$\times$1 Stage~2 (MSE+MEOM, $\lambda_{\mathrm{MSE}}{=}10$)
      & $19.30$ & $22.12$ \\
    hm$\times$1 Stage~2 (MSE+MEOM, $\lambda_{\mathrm{MSE}}{=}50$)
      & $19.20$ & $21.98$ \\
    hm$\times$1 Stage~2 (MSE+MEOM, $\lambda_{\mathrm{MSE}}{=}100$)
      & $19.11$ & $\mathbf{21.86}$ \\
    hm$\times$1 Stage~2 (MSE+MEOM, $\lambda_{\mathrm{MSE}}{=}200$)
      & $19.24$ & $22.10$ \\
    hm$\times$1 Stage~2 (MSE+MEOM, $\lambda_{\mathrm{MSE}}{=}500$)
      & $\mathbf{19.10}$ & $22.03$ \\
    \bottomrule
  \end{tabular}
\end{table}

\subsection{Computational cost of 2D-to-3D lifting}
\label{app:flops-lift}

Tab.~\ref{tab:h36m-oks-mse} in the main paper reports full-pipeline GFLOPs for the three end-to-end methods that share a ResNet-152 heatmap backbone on $384\times 384$ crops from four views.
The distinguishing computation is the lift from heatmaps (or deconvolution features) to 3D joints: confidence-weighted DLT after soft-argmax in the algebraic baseline of Iskakov et al.~\cite{iskakov2019learnable}, the same DLT after Expected-OKS decoding (Eq.~\eqref{eq:oks-soft} in the main paper) in our method, and dense unprojection followed by a V2V 3D CNN in the volumetric model~\cite{iskakov2019learnable}.
Stage~1 and Stage~2 use the same Expected-OKS\,+\,DLT inference graph, so they share the $158.2$\,GFLOPs entry.

% Queued here so the float ships at the top of the following page (D.5), not on a trailing figures-only page.
\begin{figure*}[t]
  \centering
  \includegraphics[width=\linewidth,height=0.38\textheight,keepaspectratio]{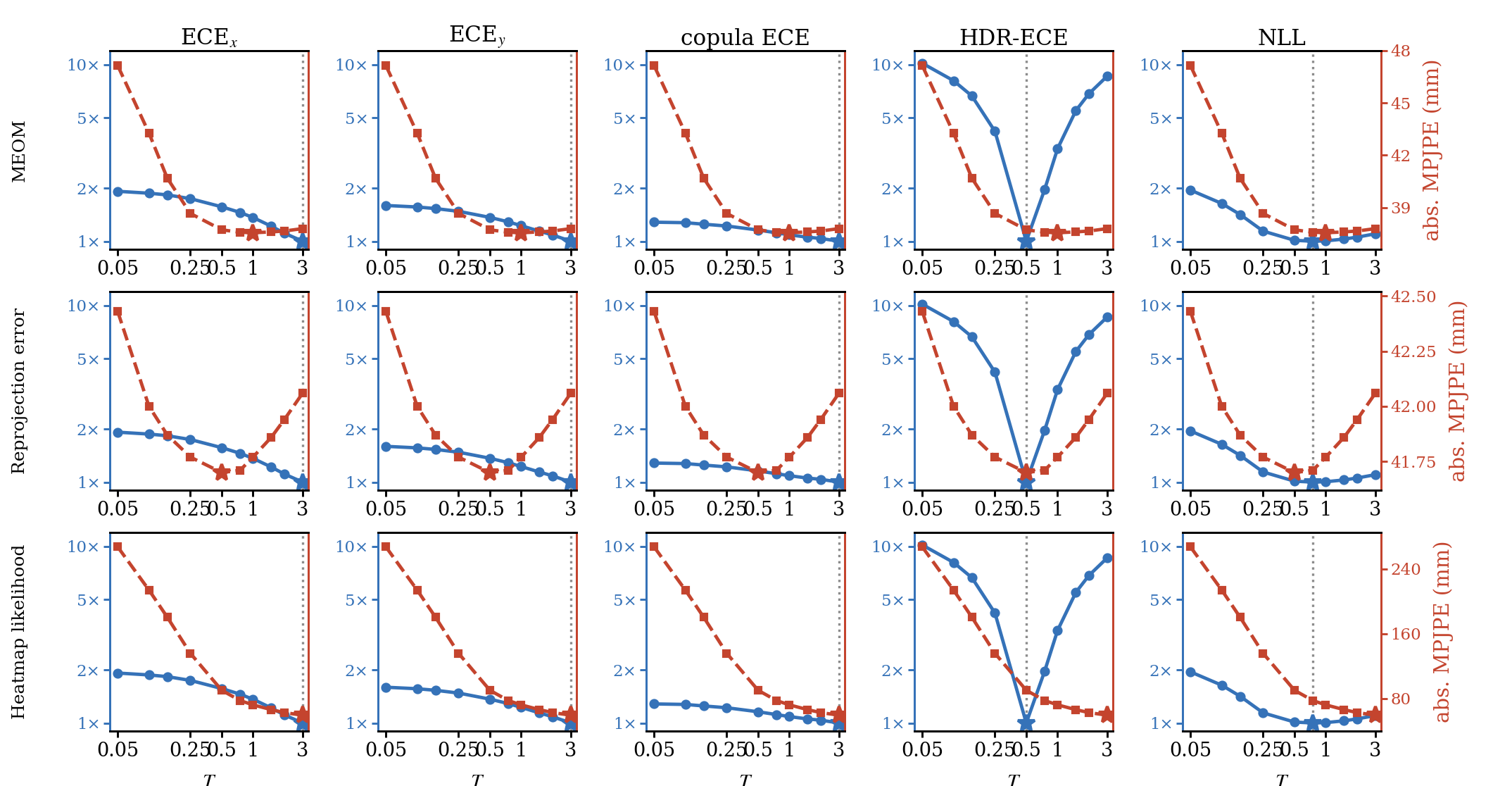}
  \caption{Same H36MA temperature sweep as Fig.~5 in the main paper, with every left axis showing the calibration score relative to its minimum on a shared log scale ($1{\times}$--$10{\times}$).
  HDR-ECE spans an order of magnitude; ECE, copula ECE, and NLL remain within about $2{\times}$ of their minima, and NLL is nearly flat near the MEOM optimum.}
  \label{fig:temp-corr-h36ma-shared}
\end{figure*}

To isolate the lift we wrap each as a standalone module, feed dummy tensors with the Human3.6M inference shapes (batch size $1$, four views, $96\times 96$ heatmaps, $64^{3}$ volume), and count multiply--adds with the fvcore flop counter.
Peak allocated GPU memory is recorded by resetting the CUDA peak tracker immediately before a forward pass and reading the subsequent maximum; latency is the mean of $30$ CUDA-event timings after warmup on an NVIDIA A100.

Tab.~\ref{tab:flops-lift} reports the lift alone and the implied full image-to-3D budget obtained by adding the shared 2D backbone ($\approx 158$\,GFLOPs for the algebraic models; $152$\,GFLOPs for volumetric triangulation, which omits the algebraic confidence head).
Analytic convolution counts agree with fvcore on the OKS kernels and on the V2V network.
Expected-OKS decoding adds a per-joint convolution of $0.162$\,GFLOPs---two orders of magnitude above algebraic soft-argmax and DLT ($1.3$\,MFLOPs), yet $919\times$ cheaper than volumetric unprojection and V2V ($149$\,GFLOPs).
Consequently the full algebraic MSE+MEOM pipeline remains at $158$\,GFLOPs, whereas volumetric triangulation reaches $301$\,GFLOPs.
Wall-clock time and memory follow the same ordering: the Expected-OKS lift runs in $1.9$\,ms with $12$\,MB peak allocation, versus $23$\,ms and $547$\,MB for the volumetric lift.

\begin{table}[h]
  \centering
  \caption{Cost of the 2D-to-3D lift with the shared ResNet-152 backbone excluded (batch size $1$, four $384\times 384$ views, $96\times 96$ heatmaps, $64^{3}$ volume). Multiply--adds are counted by fvcore. Full GFLOPs add the 2D backbone measured on the same dummy input. Latency and peak allocated memory are measured on an NVIDIA A100. Algebraic: soft-argmax\,+\,DLT. Ours: Expected-OKS convolution\,+\,temperature soft-argmax\,+\,DLT. Volumetric: $1{\times}1$ feature projection, unprojection, and V2V~\cite{iskakov2019learnable}.}
  \label{tab:flops-lift}
  \scriptsize
  \setlength{\tabcolsep}{2.5pt}
  \resizebox{\linewidth}{!}{%
  \begin{tabular}{@{}lcccc@{}}
    \toprule
    Method & Lift & Full & Latency & Peak mem. \\
     & (GFLOPs) & (GFLOPs) & (ms) & (MB) \\
    \midrule
    Algebraic (soft-argmax\,+\,DLT)
      & $0.001$ & $158.0$ & $0.74$ & $4.9$ \\
    Ours (Expected-OKS\,+\,DLT)
      & $0.162$ & $158.2$ & $1.91$ & $12.0$ \\
    Volumetric (unproject\,+\,V2V)~\cite{iskakov2019learnable}
      & $149.0$ & $301.4$ & $22.67$ & $547.3$ \\
    \bottomrule
  \end{tabular}%
  }
\end{table}

\subsection{H36MA temperature sweep}
\label{app:temp-h36ma}

Sec.~\ref{sec:temp-sweep} in the main paper uses $T\in\{0.05,0.1,0.15,0.25,0.5,0.75,1.0,1.5,2.0,3.0\}$.
For each $T$ we evaluate heatmap calibration on visible joints (visibility~$=2$): axis-wise ECE, copula ECE, HDR-ECE, NLL, and energy score; and algebraic DLT followed by 80-step refinement under the three objectives of Sec.~\ref{sec:tri-ablation} in the main paper, with view weighting on at both stages (expected OKS and reprojection error use OKS weights; heatmap likelihood uses presence).
Reported 3D errors are absolute MPJPE after refinement.
Tab.~\ref{tab:temp-h36ma} lists per-temperature calibration metrics and refined absolute MPJPE for the H36MA sweep discussed in Sec.~\ref{sec:temp-sweep} in the main paper.
Fig.~\ref{fig:temp-corr-h36ma-shared} replots the same dual-axis curves with a shared relative left axis ($1{\times}$--$10{\times}$) so that basin amplitudes can be compared across scores.

\begin{table}[h]
  \centering
  \caption{H36MA temperature sweep ($N{=}6\,238$). Inference-time ProbMapHead temperature $T$ vs.\ heatmap calibration and refined absolute MPJPE (mm) for expected OKS (ExpOKS), reprojection error (Repro), and heatmap likelihood (Lik), all with DLT/refine view weighting ON. Minimum Abs.\ MPJPE per objective in \textbf{bold}.}
  \label{tab:temp-h36ma}
  \footnotesize
  \setlength{\tabcolsep}{3.2pt}
  \begin{tabular}{@{}c ccccc ccc@{}}
    \toprule
    $T$ & ECE$_x$ & ECE$_y$ & Copula & HDR & NLL & ExpOKS & Repro & Lik \\
    \midrule
    0.05 & 0.328 & 0.354 & 0.309 & 0.265 & 7.015 & 47.15 & 42.43 & 267.24 \\
    0.10 & 0.320 & 0.347 & 0.307 & 0.211 & 5.892 & 43.27 & 42.00 & 213.90 \\
    0.15 & 0.313 & 0.340 & 0.301 & 0.173 & 5.104 & 40.68 & 41.87 & 180.41 \\
    0.25 & 0.298 & 0.328 & 0.294 & 0.109 & 4.108 & 38.66 & 41.77 & 135.88 \\
    0.50 & 0.268 & 0.303 & 0.279 & \textbf{0.026} & 3.645 & 37.72 & \textbf{41.70} & 90.30 \\
    0.75 & 0.249 & 0.286 & 0.269 & 0.051 & \textbf{3.600} & 37.58 & 41.71 & 77.67 \\
    1.00 & 0.233 & 0.272 & 0.262 & 0.087 & 3.618 & \textbf{37.55} & 41.77 & 72.30 \\
    1.50 & 0.208 & 0.253 & 0.254 & 0.143 & 3.704 & 37.60 & 41.86 & 66.39 \\
    2.00 & 0.190 & 0.241 & 0.250 & 0.178 & 3.798 & 37.65 & 41.94 & 62.79 \\
    3.00 & 0.171 & 0.222 & 0.241 & 0.224 & 3.968 & 37.79 & 42.06 & \textbf{59.55} \\
    \bottomrule
  \end{tabular}
\end{table}

%\begin{figure*}[t]
%  \centering
%  \includegraphics[width=\linewidth]{figs/temperature_correlation_h36ma.png}
%  \caption{H36MA: refined absolute MPJPE vs.\ calibration metrics across ProbMapHead temperatures (point labels~$=T$).
%  Rows: expected OKS, reprojection error, and heatmap-likelihood refinement.
%  Titles report Pearson $r$ between each calibration score and absolute MPJPE.
%  NLL aligns most strongly with expected-OKS and likelihood errors.
%  A linear coefficient understates HDR-ECE, which is unimodal in $T$ rather than monotone; on the under-dispersed branch $T{\le}0.5$ it reaches $r{=}0.938$ against expected OKS, and it is the strongest predictor of reprojection error over the full sweep ($r{=}0.896$).}
%  \label{fig:temp-corr-h36ma-full}
%\end{figure*}

\subsection{Likelihood under stronger softening}
\label{app:temp-lik-h36ma}

Because likelihood refinement continued to improve up to $T{=}3$ in Sec.~\ref{sec:temp-sweep} in the main paper, we extended the H36MA likelihood sweep to $T\in\{4,5,6,8,10,15,20\}$ (Fig.~\ref{fig:temp-lik-h36ma}).
Absolute MPJPE decreases further and plateaus near $56.4\,\mathrm{mm}$ at $T{=}20$, still ${\approx}19\,\mathrm{mm}$ above the best expected-OKS result ($37.5\,\mathrm{mm}$ at $T{=}1.0$).
Thus temperature can partially mitigate the mismatch between raw heatmap likelihood and 3D accuracy, but expected-OKS decoding remains substantially more robust and accurate under the same ProbPose predictions.

\begin{figure}[!ht]
  \centering
  \includegraphics[width=0.85\linewidth]{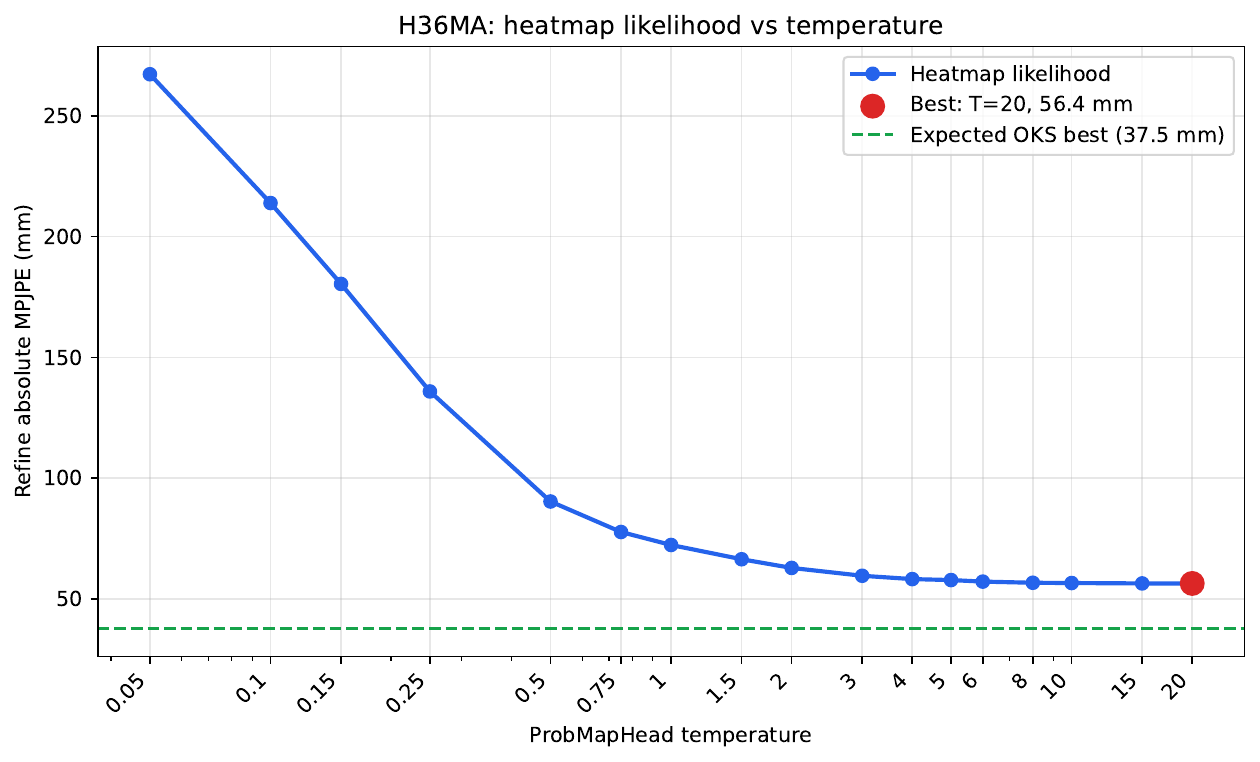}
  \caption{H36MA: heatmap-likelihood absolute MPJPE vs.\ ProbMapHead temperature (log $x$-axis), extended to $T{=}20$.
  The dashed line marks the best expected-OKS absolute MPJPE over $T{\in}[0.05,3]$ ($37.5\,\mathrm{mm}$).
  Likelihood improves with softening but saturates well above expected OKS.}
  \label{fig:temp-lik-h36ma}
\end{figure}

\end{document}